\documentclass{article}

\usepackage{microtype}
\usepackage{graphicx}
\usepackage{booktabs} % for professional tables

\usepackage{hyperref}

\usepackage[accepted]{icml2025}

\usepackage{amsmath}
\usepackage{amssymb}
\usepackage{mathtools}
\usepackage{amsthm}

\usepackage[capitalize,noabbrev]{cleveref}

\theoremstyle{plain}
\newtheorem{theorem}{Theorem}[section]

\newtheorem{lemma}[theorem]{Lemma}
\newtheorem{corollary}[theorem]{Corollary}
\theoremstyle{definition}

\theoremstyle{remark}

\usepackage[textsize=tiny]{todonotes}

\usepackage{subcaption}
\usepackage{makecell}
\renewcommand{\algorithmicrequire}{\textbf{Input:}}

\usepackage{amsmath,amsfonts,bm,amsthm,amssymb,dsfont}

\newcommand{\Eqmark}[2]{\stackrel{(#1)}{#2}}

\def\eqref#1{equation~\ref{#1}}
\def\Eqref#1{Equation~\ref{#1}}
\def\1{\bm{1}}

\def\vzero{{\bm{0}}}

\def\vb{{\bm{b}}}

\def\vu{{\bm{u}}}
\def\vv{{\bm{v}}}
\def\vw{{\bm{w}}}
\def\vx{{\bm{x}}}

\DeclareMathAlphabet{\mathsfit}{\encodingdefault}{\sfdefault}{m}{sl}
\SetMathAlphabet{\mathsfit}{bold}{\encodingdefault}{\sfdefault}{bx}{n}

\DeclareMathOperator*{\argmax}{arg\,max}
\DeclareMathOperator*{\argmin}{arg\,min}

\icmltitlerunning{Constant Stepsize Local GD for Logistic Regression: Acceleration by Instability}

\begin{document}

\twocolumn[
\icmltitle{Constant Stepsize Local GD for Logistic Regression: Acceleration by Instability}



\icmlsetsymbol{equal}{*}

\begin{icmlauthorlist}
\icmlauthor{Michael Crawshaw}{gmu}
\icmlauthor{Blake Woodworth}{gwu}
\icmlauthor{Mingrui Liu}{gmu}
\end{icmlauthorlist}

\icmlaffiliation{gmu}{Department of Computer Science, George Mason University, Fairfax, VA, USA}
\icmlaffiliation{gwu}{Department of Computer Science, George Washington University, Washington, DC, USA}

\icmlcorrespondingauthor{Michael Crawshaw}{mcrawsha@gmu.edu}
\icmlcorrespondingauthor{Mingrui Liu}{mingruil@gmu.edu}

\icmlkeywords{machine learning, optimization, distributed optimization, federated learning, edge of stability, local sgd}

\vskip 0.3in
]



\printAffiliationsAndNotice{}  % leave blank if no need to mention equal contribution

\begin{abstract}
Existing analysis of Local (Stochastic) Gradient Descent for heterogeneous objectives requires stepsizes $\eta \leq 1/K$ where $K$ is the communication interval, which ensures monotonic decrease of the objective. In contrast, we analyze Local Gradient Descent for logistic regression with separable, heterogeneous data using any stepsize $\eta > 0$. With $R$ communication rounds and $M$ clients, we show convergence at a rate $\mathcal{O}(1/\eta K R)$ after an initial unstable phase lasting for $\widetilde{\mathcal{O}}(\eta K M)$ rounds. This improves upon the existing $\mathcal{O}(1/R)$ rate for general smooth, convex objectives. Our analysis parallels the single machine analysis of~\cite{wu2024large} in which instability is caused by extremely large stepsizes, but in our setting another source of instability is large local updates with heterogeneous objectives.
\end{abstract}

\section{Introduction}
As the area of distributed optimization grows --- owing to recent applications in
federated learning \citep{mcmahan2017communication} and large-scale distributed deep
learning \citep{verbraeken2020survey} --- the gap between theory and practice has grown
proportionally. Local Stochastic Gradient Descent (SGD) and its variants have been
successfully used for distributed learning with heterogeneous data in practice for years
\citep{wang2021field, reddi2021adaptive, xu2023federated}, but so far we have little
theoretical understanding of this success \citep{wang2022unreasonable}.

The majority of theoretical works in distributed optimization take a \textit{worst-case}
approach to algorithm analysis: they consider the worst-case efficiency over some large
class of optimization problems, such as the class of convex, smooth objectives
satisfying some hetorogeneity requirement \citep{woodworth2020local,
woodworth2020minibatch, koloskova2020unified}. While the resulting guarantees are very
general, they do not always reflect practice, since they describe the worst-case, rather
than cases which may appear in practice.
%For one thing, in cases that are thoroughly understood (i.e., matching upper and lower complexity bounds), the empirically successful Local SGD is strictly dominated by a naive baseline: Minibatch SGD \citep{woodworth2020local, woodworth2020minibatch, glasgow2022sharp, patel2024limits}.
For Local SGD and its deterministic variant, Local GD, these worst-case guarantees rely
on the potentially unrealistic condition of small step sizes $\eta \leq
\mathcal{O}(1/K)$, where $K$ is the communication interval
\citep{woodworth2020minibatch, koloskova2020unified}. For Local GD, this small step size
can guarantee monotonic decrease of the objective, but such stable convergence is far
removed from practice, as non-monotonic decrease of the objective is common in practical
machine learning \citep{jastrzebski2020the, cohen2021gradient}.

Motivated by this gap between theory and practice, we take a problem-specific approach
and analyze Local GD for logistic regression. Our central question is:
\begin{center}
\textbf{\textit{Can Local GD for logistic regression achieve accelerated convergence
with a large step size ($\eta \gg 1/K$)?}}
\end{center}
Despite the apparent simplicity of this setting, existing theory is unable to answer
this question. In the single-machine setting, GD is known to converge for logistic
regression with any step size \citep{wu2024implicit, wu2024large}, and a large enough
step size will cause non-monotonic decrease of the objective. For the distributed
setting, previous work for this problem considered a two-stage variant of Local
GD~\cite{crawshaw2025local}, that uses a small step size $\eta \leq \mathcal{O}(1/K)$
before switching to a larger step size later in training. It remains open to analyze the
vanilla Local GD with a constant stepsize in this setting.

\begin{table*}[t]
\caption{Upper bounds on the objective gap $F(\vw) - F_*$ of distributed GD variants for
logistic regression, up to constants and logarithmic factors. $R$ is the number of
communication rounds, $K$ is the number of local steps, $M$ is the number of clients,
and $\gamma$ is the maximum margin of the combined dataset. $(a)$ These bounds are
derived in \citep{crawshaw2025local} by applying the worst-case upper bounds of
\citep{woodworth2020minibatch} and \citep{koloskova2020unified} to the specific problem
of logistic regression. $(b)$ Assuming $R \geq \Omega(Mn \gamma^{-2})$. $(c)$ Assuming
$R \geq \widetilde{\Omega}(\max(Mn \gamma^{-2}, KM \gamma^{-3})))$. $(d)$ This lower
bound is included for comparison of the rate in terms of $R$ and $K$, and applies to the
class of convex, $H$-smooth objectives that have a minimizer $\vw_*$ with $\|\vw_*\|
\leq B$ and $\|\nabla F_m(\vw_*) - \nabla F(\vw_*)\| \leq \zeta_*$. It should be noted
that logistic regression with separable data is not a member of this class, because no
minimizer $\vw_*$ exists for this objective.}
\label{tab:error}
\begin{center}
\begin{tabular}{@{}cccc@{}}
\toprule
& Step size & Arbitrary $K$ & Best $K$ \\
\toprule
\makecell{Local GD \\ \citep{woodworth2020minibatch}$^{(a)}$} & $\eta = \frac{1}{\gamma^{2/3} K R^{1/3}}$ & $\frac{1}{\gamma^2 KR} + \frac{1}{\gamma^{4/3} R^{2/3}}$ & $\frac{1}{\gamma^{4/3} R^{2/3}}$ \\
\hline
\makecell{Local GD \\ \citep{koloskova2020unified}$^{(a)}$} & $\eta = \frac{1}{K}$ & $\frac{1}{\gamma^2 R}$ & $\frac{1}{\gamma^2 R}$ \\
\hline
\makecell{GD \\ \citep{wu2024large}$^{(b)}$} & $\eta = \gamma^2 R$ & - & $\frac{1}{\gamma^4 R^2}$ \\
\hline
\makecell{Two-Stage Local GD \\ \citep{crawshaw2025local}} & \makecell{$\eta_1 = \frac{1}{K}$ \\ $\eta_2 = \min \left( \frac{\gamma^4 R}{KM}, 1 \right)$} & $\max \left( \frac{1}{\gamma^2 KR}, \frac{M}{\gamma^6 R^2} \right)$ & $\frac{M}{\gamma^6 R^2}$ \\
\hline
\makecell{Local GD \\ (Corollary \ref{cor:final_error})$^{(c)}$} & $\eta \in \left( 1, \frac{\gamma^3 R}{M} \right)$ & $\frac{M}{\gamma^5 R^2}$ & $\frac{M}{\gamma^5 R^2}$ \\
\toprule
\makecell{Local GD Lower Bound \\ \citep{patel2024limits}$^{(d)}$} & - & $\frac{HB^2}{R} + \frac{(H \zeta_*^2 B^4)^{1/3}}{R^{2/3}}$ & - \\
\bottomrule
\end{tabular}
\end{center}
\end{table*}

\paragraph{Contributions}
In this paper, we prove that Local GD for distributed logistic regression converges with
any step size $\eta > 0$ and any communication interval $K \geq 1$. In particular, we
show that choosing $\eta K = \widetilde{\Theta} \left( \frac{\gamma^3 R}{M} \right)$
yields a convergence rate faster than existing lower bounds of Local GD for distributed
convex optimization (see Section \ref{sec:setup} for definitions of all parameters).

Our accelerated convergence crucially uses $\eta K \gg 1$, which violates the condition
$\eta \leq \mathcal{O}(1/K)$ from previous work and potentially creates non-monotonic
objective decrease across communication rounds. To handle this instability, we adapt
techniques from the analysis of GD with large step sizes for single-machine logistic
regression, introduced by \citet{wu2024large}, which shows that GD operates in an
initial unstable phase before entering a stable phase where the objective decreases
monotonically. We use these techniques to analyze Local GD by decomposing the
algorithm's update into the contribution from each individual data point, and
tracking this contribution throughout the local update steps, in order to relate
the trajectory of Local GD to that of GD. Consequently, we can show that Local
GD also transitions from an unstable phase to a stable phase.

We also experimentally evaluate Local GD for logistic regression with synthetic data and
MNIST data, and the results corroborate our theoretical finding that acceleration can be
achieved by allowing for non-monotonic objective decrease. To probe the limitations of
our theory, we evaluate Local GD under different regimes of $\eta$ and $K$, and
accordingly we propose open problems and directions for future research.

\paragraph{Organization} We first discuss related work (Section \ref{sec:related_work}),
then state our problem (Section \ref{sec:setup}) and give our analysis (Section
\ref{sec:analysis}). We provide experimental results (Section \ref{sec:experiments}),
then conclude with a discussion of our results and future work (Section
\ref{sec:discussion}).

\paragraph{Notation}
For $n \in \mathbb{N}$, we denote $[n] = \{1, \ldots, n\}$. We use $\|\cdot\|$ to denote
the $L_2$ norm for vectors and the spectral norm for matrices. Outside of the abstract,
we use $\mathcal{O}$, $\Omega$, and $\Theta$ to omit only universal constants.
Similarly, $\widetilde{\mathcal{O}}$, $\widetilde{\Omega}$, and $\widetilde{\Theta}$
only omit universal constants and logarithmic terms.

\section{Related Work} \label{sec:related_work}

\paragraph{General Distributed Optimization}
Early work in this area focused on distributed algorithms for solving classical learning
problems with greater efficiency through parallelization \citep{mcdonald2009efficient,
mcdonald2010distributed, zinkevich2010parallelized, dekel2012optimal,
balcan2012distributed, zhang2013information, shamir2014distributed,
arjevani2015communication}. Recent years have seen a growth of research in distributed
optimization due to applications for large-scale distributed training of neural networks
\citep{tang2020communication, verbraeken2020survey} and federated learning
\citep{mcmahan2017communication}. Federated learning is a paradigm for distributed
learning in which user devices collaboratively train a machine learning model without
sharing data; see \citep{kairouz2021advances, wang2021field} for a comprehensive survey.

\paragraph{Efficiency of Local SGD}
Local SGD (also known as Federated Averaging, or FedAvg) is a fundamental algorithm for
distributed optimization, both in theory and practice. Convergence guarantees of Local
SGD for distributed convex optimization under various conditions were proven by
\citet{stich2019local, haddadpour2019convergence, woodworth2020minibatch,
khaled2020tighter, koloskova2020unified, glasgow2022sharp}. These works consider the
worst-case efficiency of Local SGD for solving large classes of optimization problems,
such as the class of problems with smooth, convex objectives with some condition on the
heterogeneity between local objectives; we refer to these guarantees as
\textit{worst-case baselines}. Lower bounds have established that Local SGD is dominated
by Minibatch SGD in the worst case over various problem classes despite the fact that
Local SGD tends to outperform Minibatch SGD for practical problems
\citep{woodworth2020local, woodworth2020minibatch, glasgow2022sharp, patel2024limits},
and variants of Local SGD remain standard in practice \citep{wang2021field,
wang2022unreasonable, reddi2021adaptive, xu2023federated}. It remains an active topic of
research to develop a theoretical understanding of Local SGD and Minibatch SGD that
aligns with practical observations \citep{woodworth2020minibatch, glasgow2022sharp,
wang2022unreasonable, patel2023on, patel2024limits}.

\paragraph{Gradient Methods for Logistic Regression}
In this work, we narrow our focus and consider the efficiency of Local GD for solving
one particular optimization problem, continuing a line of work which shows that
gradient-based optimization algorithms have very particular behavior for certain
problems of interest in machine learning. \citet{soudry2018implicit, ji2019implicit}
showed that GD for logistic regression converges to the maximum margin solution without
explicit regularization. \citet{gunasekar2018characterizing, nacson2019stochastic,
ji2021fast} proved further implicit regularization results for general steepest descent
methods, stochastic gradient descent, and a fast momentum-based algorithm, respectively. %\blake{Depending on space, I'm not sure that we necessarily need to mention the implicit regularization work---it's not that closely related.}
A separate line of work observed that GD exhibits non-monotonic decrease in the
objective when training neural networks, a phenomenon called the Edge of Stability
\citep{cohen2021gradient, damian2023selfstabilization}.

The works which are most closely related to ours are \citep{wu2024implicit, wu2024large}
and \citep{crawshaw2025local}. \citet{wu2024implicit} showed that GD for logistic
regression can converge with any positive stepsize, despite non-monotonic decrease of
the objective, and that GD converges to the maximum margin solution. \citet{wu2024large}
showed that GD with a large stepsize can achieve accelerated convergence for logistic
regression. \citet{crawshaw2025local} proved that a two-stage variant of Local GD can
achieve accelerated convergence compared to the worst-case baselines
\cite{koloskova2020unified, woodworth2020minibatch}.

\begin{algorithm}[t]
    \caption{Local GD}
    \label{alg:local_gd}
    \begin{algorithmic}[1]
        \REQUIRE Initialization $\vw_0 \in \mathbb{R}^d$, rounds $R \in \mathbb{N}$, local steps $K \in \mathbb{N}$, learning rate $\eta > 0$
        \FOR{$r = 0, 1, \ldots, R-1$}
            \FOR{$m \in [M]$}
                \STATE $\vw_{r,0}^m \gets \vw_r$
                \FOR {$k = 0, \ldots, K-1$}
                    \STATE $\vw_{r,k+1}^m \gets \vw_{r,k}^m - \eta \nabla F_m(\vw_{r,k}^m)$
                \ENDFOR
            \ENDFOR
            \STATE $\vw_{r+1} \gets \frac{1}{M} \sum_{m=1}^M \vw_{r,K}^m$
        \ENDFOR
    \end{algorithmic}
\end{algorithm}

\section{Problem Setup} \label{sec:setup}
We consider a distributed version of binary classification with linearly separable data.
The number of clients is denoted by $M$, the number of data points per client as $n$,
and the dimension of the input data as $d$. The data consists of $M$ local datasets, one
for each client: $D_m = \{(\vx_i^m, y_i^m)\}_{i \in [n]}$ for each $m \in [M]$, where
$\vx_i^m \in \mathbb{R}^d$ and $y_i^m \in \{-1, 1\}$. We assume that the global dataset
$D = \cup_{m \in [M]} D_m$ is linearly separable, that is, there exists some $\vw \in
\mathbb{R}^d$ such that $y \langle \vw, \vx \rangle > 0$ for every $(x, y) \in D$. We
also denote by $\gamma$ and $\vw_*$ the maximum margin and the maximum margin classifier
for the global dataset, that is,
\begin{align}
    \gamma &= \max_{\vw \in \mathbb{R}^d, \|\vw\| = 1} \min_{(x, y) \in D} y \langle \vw, \vx \rangle \\
    \vw_* &= \argmax_{\vw \in \mathbb{R}^d, \|\vw\| = 1} \min_{(x, y) \in D} y \langle \vw, \vx \rangle.
\end{align}
Note that $\gamma > 0$ from the assumption of linear separability.

We are interested in studying the behavior of Local Gradient Descent (Algorithm
\ref{alg:local_gd}) for minimizing the logistic loss of this classification problem.
Denoting $\ell(z) = \log(1 + \exp(-z))$, the local objective $F_m: \mathbb{R}^d
\rightarrow \mathbb{R}$ for client $m \in [M]$ is defined as
\begin{equation}
    F_m(\vw) = \frac{1}{n} \sum_{i=1}^n \ell(y_i^m \langle \vw, \vx_i^m \rangle),
\end{equation}
and our goal is to approximately solve the following:
\begin{equation} \label{eq:opt_prob}
    \min_{\vw \in \mathbb{R}^d} \left\{ F(\vw) := \frac{1}{M} \sum_{m=1}^M F_m(\vw) \right\}.
\end{equation}
In this work, we focus on minimization of this training loss, and guarantees for the
population loss can be derived using standard techniques.

Notice that the objective depends on each data point $(\vx_i^m, y_i^m)$ only through
the product $y_i^m \vx_i^m$. Therefore, we can assume without loss of generality
that $y_i^m = 1$ for every $m \in [M], i \in [n]$, since we can replace any data point
$(\vx_i^m, -1)$ with $(-\vx_i^m, 1)$, which preserves the product $y_i^m
\vx_i^m$ and therefore does not change the trajectory of Local GD. We also assume that
$\|\vx_i^m\| \leq 1$ for every $m, i$, which can always be enforced by rescaling all
data points by $\max_{m, i} \|\vx_i^m\|$. Lastly, we will denote by $H$ the smoothness
constant of $F$, that is, $H := \sup_{\vw \in \mathbb{R}^d} \|\nabla^2 F(\vw)\|$, which
satisfies $H \leq 1/4$ when $\|\vx_i^m\| \leq 1$ \citep{crawshaw2025local}.

\section{Convergence Analysis} \label{sec:analysis}
We present two convergence results of Local GD for the logistic regression problem
stated in \Eqref{eq:opt_prob}. Our Theorem \ref{thm:stage1} gives an upper bound on the
average objective $\frac{1}{r} \sum_{s=0}^{r-1} F(\vw_r)$ over the first $r$
communication rounds, which holds for any $r$. On the other hand, Theorem
\ref{thm:stage2} provides a last-iterate upper bound on the objective $F(\vw_r)$ for
every $r$ after a transition time $\tau$. Both of these results hold for any learning
rate $\eta > 0$ and any number of local steps $K$. Corollary \ref{cor:final_error}
summarizes our results by deriving the error with the best choices of $\eta$ and $K$ for
a given communication budget $R$. We first state and discuss the results in Section
\ref{sec:statement}, then give an overview of the proofs in Section \ref{sec:proofs}.
The complete proofs are deferred to Appendix \ref{app:proofs}.

\subsection{Statement of Results} \label{sec:statement}
Theorems \ref{thm:stage1} and \ref{thm:stage2} provide guarantees in two phases: the
initial unstable phase (lasting for $\tau$ rounds), and the latter stable phase. During
the unstable phase, we cannot provide a last-iterate guarantee, but we can upper bound
the average loss over the trajectory. After the loss becomes sufficiently small, Local
GD enters the stable phase, where the loss decreases monotonically at every round. These
two phases mimic the observed behavior of Local GD in experiments (see Section
\ref{sec:experiments}), and align with the behavior of single-machine GD
\citep{wu2024large}.

\begin{theorem} \label{thm:stage1}
For every $r \geq 0$, Local GD satisfies
\begin{align}
    &\frac{1}{r} \sum_{s=0}^{r-1} F(\vw_s) \leq \nonumber \\
    &\quad 26 \frac{\|\vw_0\|^2 + 1 + \log^2(K + \eta K \gamma^2 r) + \eta^2 K^2}{\eta \gamma^4 r}. \label{eq:unstable_ub}
\end{align}
\end{theorem}

Notice that the RHS of \Eqref{eq:unstable_ub} grows at most linearly with $\eta$ and
quadratically with $K$: this aligns with the intuition that large stepsizes and/or long
communication intervals can create instability. Indeed, even if $\eta \leq 1/H$, so that
the local objectives are guaranteed to decrease with each local step, the global
objective may not decrease monotonically over rounds when $K$ is large, due to a large
effective per-round step size $\eta K$. However, for any fixed $\eta$ and $K$, Theorem
\ref{thm:stage1} shows that the average loss can be made arbitrarily small with large
enough $r$. After at most $\tau$ rounds, $F(\vw_r)$ will decrease below a certain
threshold, after which the global objective will decrease monotonically with each
communication round, leading to the following last-iterate guarantee.

\begin{theorem} \label{thm:stage2}
Denote $\psi = \min \left( \frac{\gamma}{140 \eta KM}, \frac{1}{2Mn} \right)$ and
\begin{equation}
    \tau = \frac{4 \gamma \|\vw_0\| + 2 \sqrt{2} + 2 \eta + \log \left( 1 + \frac{\sqrt{K}}{\sqrt{\eta} \gamma \psi} \right)}{\eta \gamma^2 \psi}.
\end{equation}
For every $r \geq \tau$, Local GD satisfies
\begin{equation} \label{eq:stage2}
    F(\vw_r) \leq \frac{16}{\eta \gamma^2 K (r-\tau)}.
\end{equation}
\end{theorem}

Note that Theorems \ref{thm:stage1} and \ref{thm:stage2} apply for any choice of the
stepsize $\eta$ and number of local steps $K$. In contrast with the worst-case analysis
which requires that $\eta \leq \mathcal{O} \left( \frac{1}{K} \right)$, ours is the
first result showing that Local GD can converge for logistic regression without any
restrictions on $\eta$ and $K$. The following corollary shows that, by tuning $\eta$ and
$K$, we can achieve an accelerated rate with $R^{-2}$ dependence on $R$, which improves
upon the lower bounds of Local GD for general distributed convex optimization (see Table
\ref{tab:error}).

\begin{corollary} \label{cor:final_error}
Suppose $R \geq \widetilde{\Omega} \left( \max \left( \frac{Mn}{\gamma^2},
\frac{KM}{\gamma^3} \right) \right)$. With $\vw_0 = \vzero$, $\eta \geq 1$, and $\eta K =
\widetilde{\Theta} \left( \frac{\gamma^3 R}{M} \right)$, Local GD satisfies
\begin{equation}
    F(\vw_R) \leq \widetilde{\mathcal{O}} \left( \frac{M}{\gamma^5 R^2} \right).
\end{equation}
\end{corollary}

The condition $R \geq \widetilde{\Omega} \left( \max \left( \frac{Mn}{\gamma^2},
\frac{MK}{\gamma^3} \right) \right)$ ensures that $R \geq \tau$, so that training will
actually enter the stable phase and decrease the objective at the rate $1/(\eta \gamma^2
KR)$. A similar condition is used in the analysis of GD with large stepsizes for
single-machine logistic regression \citep{wu2024large}.

Also, note that aside from the condition $\eta \geq 1$, the stepsize $\eta$ and the
communication interval $K$ always appear together as the product $\eta K$. This means
that our guarantee does not distinguish the performance of Local GD as $K$ changes, so
long as the stepsize changes to keep $\eta K$ constant. Therefore, it remains open to
show whether or not Local GD can actually benefit from the use of local steps for this
problem. Indeed, the analysis of GD for single-machine logistic regression
\citep{wu2024large} immediately implies that for our distributed problem, GD
(parallelized over $M$ machines) achieves error $\widetilde{\mathcal{O}}(1/(\gamma^4
R^2))$, which improves upon our guarantee for Local GD in terms of $M$ and $1/\gamma$.
We further discuss this comparison in Section \ref{sec:discussion}.

\subsection{Proof Overview} \label{sec:proofs}
Throughout the analysis, we will denote $b_{r,i}^m = \langle \vw_r, \vx_i^m \rangle$, so
that $F_m(\vw_r) = \frac{1}{n} \sum_{i=1}^n \ell(b_{r,i}^m)$. Similarly, we will denote
$b_{r,i,k}^m = \langle \vw_{r,k}^m, \vx_i^m \rangle$.

The proofs of Theorems \ref{thm:stage1} and \ref{thm:stage2} adapt existing tools
introduced by \cite{wu2024large} and \cite{crawshaw2025local}; our application of these
tools for our setting relies on a comparison between the trajectories of GD and Local GD
by decomposing updates into the contribution from each individual data point $\vx_i^m$.
Specifically, a single GD update starting from $\vw_r$ is
\begin{equation} \label{eq:gd_update_decomp}
    -\eta \nabla F(\vw_r) = \frac{\eta}{Mn} \sum_{m=1}^M \sum_{i=1}^n |\ell'(b_{r,i}^m)| \vx_i^m.
\end{equation}
Denoting
\begin{equation}
    \beta_{r,i}^m = \frac{\frac{1}{K} \sum_{k=0}^{K-1} |\ell'(b_{r,i,k}^m)|}{|\ell'(b_{r,i}^m)|},
\end{equation}
a single round update of Local GD from $\vw_r$ can be rewritten
\begin{align}
    \vw_{r+1} - \vw_r &= -\frac{\eta}{M} \sum_{m=1}^M \sum_{k=0}^{K-1} \nabla F_m(\vw_{r,k}^m) \\
    &= \frac{\eta K}{Mn} \sum_{m=1}^M \sum_{i=1}^n \beta_{r,i}^m |\ell'(b_{r,i}^m)| \vx_i^m. \label{eq:local_gd_update_decomp}
\end{align}
Comparing \Eqref{eq:gd_update_decomp} and \Eqref{eq:local_gd_update_decomp}, the updates
for GD and Local GD can both be represented as linear combinations of the data
$\vx_i^m$, and the two trajectories can be compared by analyzing the coefficients
$\beta_{r,i}^m$. By upper and lower bounding $\beta_{r,i}^m$, we can adapt the split
comparator and gradient potential techniques of \citet{wu2024large} (which were
introduced for GD) to analyze Local GD during the unstable phase and show a transition
to stability.

For the stable phase, we leverage the relationship between the derivatives of the
objective function, namely that
\begin{equation}
    \|\nabla^2 F(\vw)\| \leq F(\vw) \quad \text{and} \quad \|\nabla F(\vw)\| \leq F(\vw),
\end{equation}
to show that a small objective value $F(\vw)$ implies a small local smoothness
$\|\nabla^2 F(\vw')\|$ for $\|\vw' - \vw\| \leq 1$, and this in turn implies monotonic
decrease of the objective. A similar argument was used by \citet{crawshaw2025local}, but
here we use a refined version that allows for any $\eta > 0$, whereas the analysis of
\citet{crawshaw2025local} requires $\eta \leq 1/H$.

Below we state key lemmas to sketch the proofs of each theorem, and full proofs are
deferred to Appendix \ref{app:proofs}.

\paragraph{Unstable Phase}
As previously mentioned, we aim to apply the split comparator technique of
\citet{wu2024large} to analyze Local GD, and we can do so if we upper and lower bound
$\beta_{r,i}^m$. Our lower bound is surprisingly simple:
\begin{equation}
    \beta_{r,i}^m = \frac{\frac{1}{K} \sum_{k=0}^{K-1} |\ell'(b_{r,i,k}^m)|}{|\ell'(b_{r,i}^m)|} \geq \frac{1}{K},
\end{equation}
where the inequality simply ignores all terms of the sum in the numerator, except that
corresponding to $k=0$. While this may appear very loose, it is not hard to show in
special cases that this bound is tight up to logarithmic factors for certain values of
$\vw_r$ (see Lemma \ref{lem:beta_ub}).

We upper bound $\beta_{r,i}^m$ as
\begin{align}
    \beta_{r,i}^m &= \frac{1}{K} \sum_{k=0}^{K-1} \frac{1 + \exp(b_{r,i}^m)}{1 + \exp(b_{r,i,k}^m)} \\
    &\leq 1 + \exp(b_{r,i}^m) = 1 + \exp(\langle \vw_r, \vx_i^m) \rangle \\
    &\leq 1 + \exp(\|\vw_r\|), \label{eq:sketch_beta_ub_inter}
\end{align}
where the last line uses $\|\vx_i^m\| \leq 1$. To bound $\|\vw_r\|$, we apply the
split comparator technique of \citet{wu2024large} to analyze the local trajectories of
each round $\{\vw_{s,k}^m\}_k$, then use this to establish a recursive bound on $\|\vw_s
- \vu\|$ over rounds, where $\vu = \vu_1 + \vu_2$ is a yet unspecified comparator. The
analysis within each round implies that
\begin{align}
    &\frac{\|\vw_{s,K}^m - \vu\|^2}{2 \eta K} + \frac{1}{K} \sum_{k=0}^{K-1} F_m(\vw_{s,k}^m) \leq \nonumber \\
    &\quad \frac{\|\vw_s - \vu\|^2}{2 \eta K} + F_m(\vu_1),
\end{align}
and in particular that
\begin{equation}
    \|\vw_{s,K}^m - \vu\| \leq \|\vw_s - \vu\| + \sqrt{2 \eta K F_m(\vu_1)}.
\end{equation}
Averaging over $m \in [M]$ and recursing over $s \in \{0, \ldots, r-1\}$ implies that
\begin{equation}
    \|\vw_r - \vu\| \leq \|\vw_0 - \vu\| + r \sqrt{2 \eta K F(\vu_1)},
\end{equation}
so
\begin{equation}
    \|\vw_r\| \leq \|\vw_0\| + 2 \|\vu\| + r \sqrt{2 \eta K F(\vu_1)}.
\end{equation}
By choosing $\vu$ to balance the last two terms on the RHS, we arrive at the following
bound.

\begin{lemma} \label{lem:param_ub}
For every $r \geq 0$,
\begin{equation}
    \|\vw_r\| \leq \|\vw_0\| + \frac{\sqrt{2} + \eta + \log(1 + \eta \gamma^2 Kr^2)}{\gamma}.
\end{equation}
\end{lemma}

We can now plug this in to \Eqref{eq:sketch_beta_ub_inter} to upper bound
$\beta_{r,i}^m$. Although the bound for $\beta_{r,i}^m$ is exponential in $\|\vw_r\|$,
Lemma \ref{lem:param_ub} shows that $\|\vw_r\|$ is only logarithmic in $r$, so the
resulting upper bound of $\beta_{r,i}^m$ is only polynomial in $r$.

With upper and lower bounds of $\beta_{r,i}^m$, the split comparator technique can be
used to analyze Local GD similarly as for GD. The full proof can be found in Appendix
\ref{app:unstable_proofs}.

\paragraph{Stable Phase}
Our error bound for the stable phase uses the following modified descent
inequality:
\begin{lemma} \label{lem:descent}
For $\vw, \vw' \in \mathbb{R}^d$, if $\|\vw' - \vw\| \leq 1$, then for every $m \in [M]$,
\begin{align}
    &F_m(\vw') - F_m(\vw) \leq \\
    &\quad F_m(\vw) + \langle \nabla F_m(\vw), \vw' - \vw \rangle + 4 F_m(\vw) \|\vw' - \vw\|^2. \nonumber
\end{align}
\end{lemma}
The above descent inequality is proven by using the facts that $\|\nabla^2 F_m(\vw)\|
\leq F_m(\vw)$ (Lemma \ref{lem:obj_grad_ub}), and $\|\vw' - \vw\| \leq \mathcal{O}(1)$
implies that $\|\nabla^2 F_m(\vw')\| \leq \mathcal{O}(\|\nabla ^2 F_m(\vw)\|)$ (Lemma
\ref{lem:hessian_gronwall_ub}). This descent inequality captures a desirable property of
the logistic loss: the local smoothness constant decreases with the objective value, so
that large stepsizes can yield monotonic objective decrease as long as the objective is
below some threshold.

To use this lemma to bound the error of Local GD, we need to do three things: (1) show
that $\|\vw_{r+1} - \vw_r\| \leq 1$ when $F(\vw_r)$ is below some threshold; (2) show
that the bias in the update direction $\vw_{r+1} - \vw_r$ compared to $-\eta K \nabla
F(\vw_r)$ is negligible when $F(\vw_r)$ is below some threshold; (3) show that
$F(\vw_r)$ becomes smaller than our desired threshold within $\tau$ rounds.

First, to show that $\|\vw_{r+1} - \vw_r\| \leq 1$ based on the magnitude of $F(\vw_r)$,
notice
\begin{align}
    \|\vw_{r+1} - \vw_r\| &= \eta \left\| \frac{1}{M} \sum_{m=1}^M \sum_{k=0}^{K-1} \nabla F_m(\vw_{r,k}^m) \right\| \\
    &\leq \frac{\eta}{M} \sum_{m=1}^M \sum_{k=0}^{K-1} \|\nabla F_m(\vw_{r,k}^m)\|. \label{eq:sketch_movement_bound}
\end{align}
We know $\|\nabla F_m(\vw_{r,k}^m)\| \leq F_m(\vw_{r,k}^m)$ (Lemma
\ref{lem:obj_grad_ub}), and if we knew that local updates monotonically decrease the
local loss, we further have $F_m(\vw_{r,k}^m) \leq F_m(\vw_r)$. Combined with
\Eqref{eq:sketch_movement_bound}, this would yield
\begin{align}
    \|\vw_{r+1} - \vw_r\| \leq \eta K F(\vw_r).
\end{align}
In fact, we can use Lemma \ref{lem:descent} to show that local updates monotonically
decrease the local objective, that is, $F_m(\vw_{r,k+1}^m) \leq F_m(\vw_{r,k}^m)$, whenever
$F_m(\vw_{r,k}^m) \leq 1/(4 \eta)$. This shows that local objectives monotonically decrease
across local steps (Lemma \ref{lem:local_stability}), and this in turn implies that
$\|\vw_{r,k}^m - \vw_r\| \leq 1$ (Lemma \ref{lem:bounded_movement}).

\begin{lemma} \label{lem:local_stability}
If $F(\vw_r) \leq 1/(4 \eta M)$ for some $r \geq 0$, then $F_m(\vw_{r,k}^m)$ is decreasing in $k$ for every $m \in [M]$.
\end{lemma}

\begin{lemma} \label{lem:bounded_movement}
If $F(\vw_r) \leq 1/(\eta KM)$ for some $r \geq 0$, then $\|\vw_{r,k}^m -
\vw_r\| \leq 1$ for every $m \in [M], k \in [K]$.
\end{lemma}

By choosing $k=K$ and averaging over $m \in [M]$, Lemma \ref{lem:bounded_movement}
implies that $\|\vw_{r+1} - \vw_r\| \leq 1$.

Next, to handle the bias of the update direction, we rewrite the update as
\begin{equation}
    \vw_{r+1} - \vw_r = -\eta K (\nabla F(\vw_r) + \vb_r),
\end{equation}
where
\begin{equation}
    \vb_r = \frac{1}{MK} \sum_{m=1}^M \sum_{k=0}^{K-1} (\nabla F_m(\vw_{r,k}^m) - \nabla F_m(\vw_r)).
\end{equation}
We can bound the magnitude of the bias as follows:
\begin{equation}
    \|\vb_r\| \leq \frac{1}{MK} \sum_{m=1}^M \sum_{k=0}^{K-1} \|\nabla F_m(\vw_{r,k}^m) - \nabla F_m(\vw_r)\|,
\end{equation}
and denoting $C = \{(1-t) \vw_r + t \vw_{r,k}^m \;|\; t \in [0, 1]\},$
\begin{align}
    &\|\nabla F_m(\vw_{r,k}^m) - \nabla F_m(\vw_r)\| \\
    &\quad \leq \left( \max_{\vw \in C} \|\nabla^2 F_m(\vw)\| \right) \|\vw_{r,k}^m - \vw_r\| \\
    &\quad \leq \left( \max_{\vw \in C} F_m(\vw) \right) \|\vw_{r,k}^m - \vw_r\| \\
    &\quad \leq \max \left( F_m(\vw_r), F_m(\vw_{r,k}^m) \right) \|\vw_{r,k}^m - \vw_r\|, \label{eq:sketch_hetero_inter}
\end{align}
where the last two inequalities use $\|\nabla^2 F_m(\vw)\| \leq F_m(\vw)$ (Lemma
\ref{lem:obj_grad_ub}) and convexity of $F_m$, respectively. Using Lemmas
\ref{lem:local_stability} and \ref{lem:bounded_movement}, we can already bound the two
terms of \Eqref{eq:sketch_hetero_inter} when $F_m(\vw_r)$ is small, which gives the
following.

\begin{lemma} \label{lem:update_bias_ub}
If $F(\vw_r) \leq \gamma/(70 \eta K M)$, then $\|\vb_r\| \leq \frac{1}{5} \|\nabla F(\vw_r)\|$.
\end{lemma}

Third, we must show that $F(\vw_r)$ will be sufficiently small for some $r \leq \tau$ in
order to satisfy the conditions of Lemmas \ref{lem:local_stability},
\ref{lem:bounded_movement}, and \ref{lem:update_bias_ub}. To do this, we adapt the
gradient potential argument of \citet{wu2024large}, as previously mentioned, by lower
bounding $\beta_{r,i}^m$. We use the same bound as in the proof of Theorem
\ref{thm:stage1}: $\beta_{r,i}^m \geq 1/K$. This allows us to relate the gradient
potential of Local GD to that of GD, and combining this with Lemma \ref{lem:param_ub}
shows that $F(\vw_r)$ is sufficiently small to enable stable descent after $\tau$
rounds.

\begin{lemma} \label{lem:transition_time}
There exists some $r \leq \tau$ such that $F(\vw_r) \leq \frac{\gamma}{70 \eta K M}$.
\end{lemma}

Finally, to prove Theorem \ref{thm:stage2}, we can apply Lemma \ref{lem:descent} for all
$r \geq \tau$. Applying Lemma \ref{lem:update_bias_ub} to control the bias of the update
direction, we obtain
\begin{equation}
    F(\vw_{r+1}) - F(\vw_r) \leq -\frac{1}{4} \eta K \|\nabla F(\vw_r)\|^2.
\end{equation}
Using $\|\nabla F(\vw_r)\| \geq \frac{\gamma}{2} F(\vw_r)$ (Lemma
\ref{lem:obj_grad_ub}), this leads to a recursion over $F(\vw_r)$, and unrolling back to
round $\tau$ gives exactly \Eqref{eq:stage2} from Theorem \ref{thm:stage2}. The full
proof is given in Appendix \ref{app:stable_proofs}.

Corollary \ref{cor:final_error}, which gives our result stated in Table \ref{tab:error},
is proved in Appendix \ref{app:corollary_proof}.

\subsection{Comparison to Single-Machine Case} When $K=1$ or $M=1$, the Local GD
algorithm reduces to GD. However, our convergence rate of $M/(\gamma^5 R^2)$ does not
exactly recover the $1/(\gamma^4 R^2)$ rate of \citet{wu2024large} in terms of the
dataset's margin $\gamma$. Here we provide some technical details on the origin of this
issue and whether it can be removed.

The issue of our $\gamma$ dependence stems from bounding the bias term $\lVert \vb_r
\rVert$ in Lemma \ref{lem:update_bias_ub}. $\vb_r$ is the difference between the update
direction for a round compared to the global gradient at the beginning of that round.
Notice that other conditions for entering the stable phase (Lemma
\ref{lem:local_stability}, Lemma \ref{lem:bounded_movement}) only require $F(\vw_r) \leq
O(1/(\eta KM))$, whereas Lemma \ref{lem:update_bias_ub} requires $F(\vw_r) \leq
O(\gamma/(\eta KM))$. This additional factor of $\gamma$ needed to bound $\lVert \vb_r
\rVert$ creates the worse dependence on $\gamma$ compared with the single-machine case.
Note that the gradient bias results from taking multiple local steps before averaging,
so it does not appear when $K=1$ or $M=1$.

Technically, the requirement $F(\vw_r) \leq O(\gamma/(\eta KM))$ might be weakened, but
with a fine-grained analysis of the Local GD trajectory. First, note that the
requirement on $F(\vw_r)$ is used in \Eqref{eq:update_bias_ub_inter} of Lemma
\ref{lem:app_update_bias_ub}, for the inequality marked $(iv)$. The need for the factor
of $\gamma$ arises from the next inequality (marked $(v)$), where we apply $F(\vw) \leq
2 \lVert \nabla F(\vw) \rVert / \gamma$ (Lemma \ref{lem:obj_grad_lb}). The additional
factor of $\gamma$ is needed to cancel out the $1/\gamma$ from Lemma
\ref{lem:obj_grad_lb}. Now, if we had a stronger bound in Lemma \ref{lem:obj_grad_lb}
--- say $F(\vw) \leq \lVert \nabla F(\vw) \rVert$ --- then we could remove the extra
$\gamma$ factor. The bound $F(\vw) \leq \lVert \nabla F(\vw) \rVert$ does not hold for
all $\vw$, but it does hold for some $\vw$, namely in the case that $\vw = t \vw_*$,
where $t$ is a large scalar. So we could possibly improve the gamma dependence if we
knew that Local GD converges to the max-margin solution, however, this kind of implicit
bias of Local GD with large $\eta$ or $K$ is not known; even in the single-machine case
the implicit bias of GD for logistic regression is unknown when the step size scales
linearly with the number of iterations \citep{wu2024large}. We consider this implicit
bias analysis as an important direction of future research.

\begin{figure*}[ht]
\vskip 0.2in
\begin{center}
\begin{subfigure}[b]{0.49\textwidth}
    \includegraphics[width=\columnwidth]{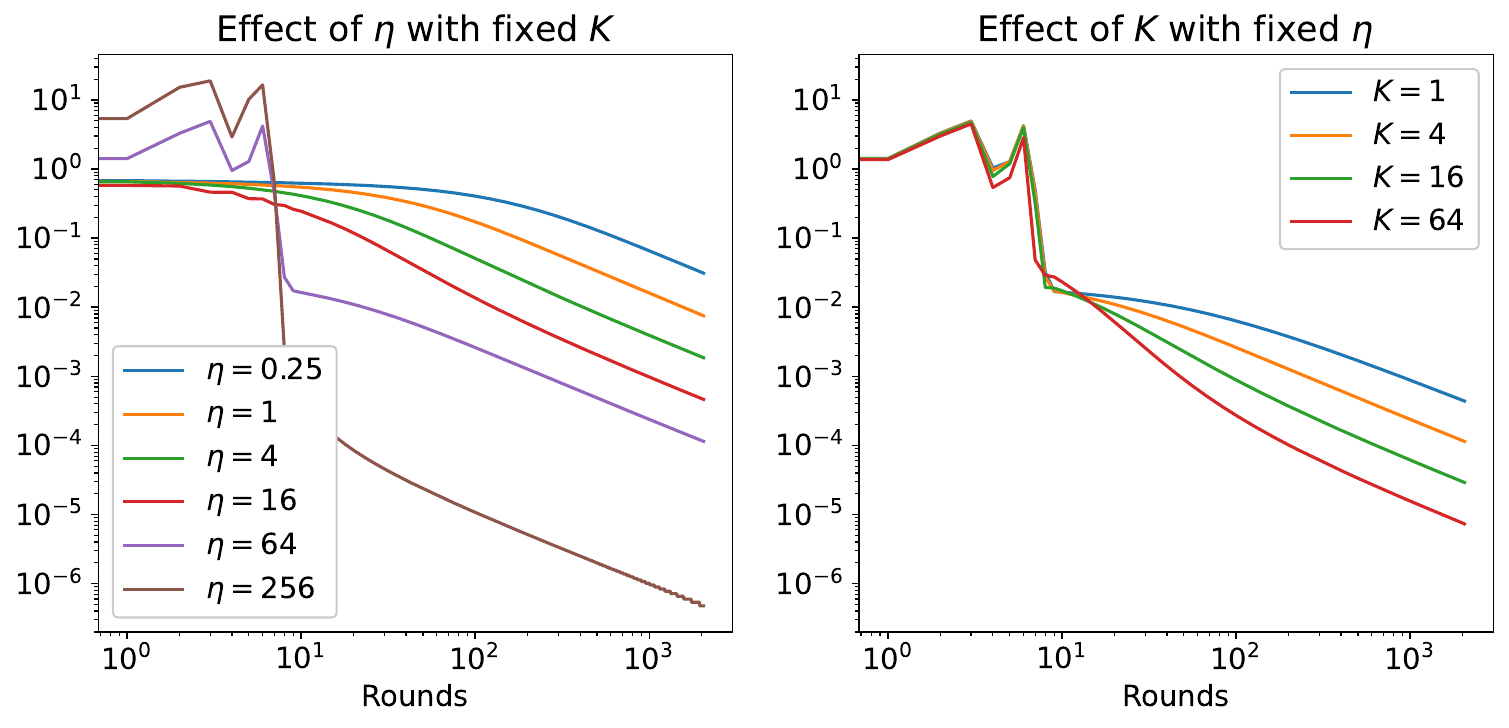}
    \caption{Synthetic data. Left: $K = 4$. Right: $\eta = 64$.}
\end{subfigure}
\begin{subfigure}[b]{0.49\textwidth}
    \includegraphics[width=\columnwidth]{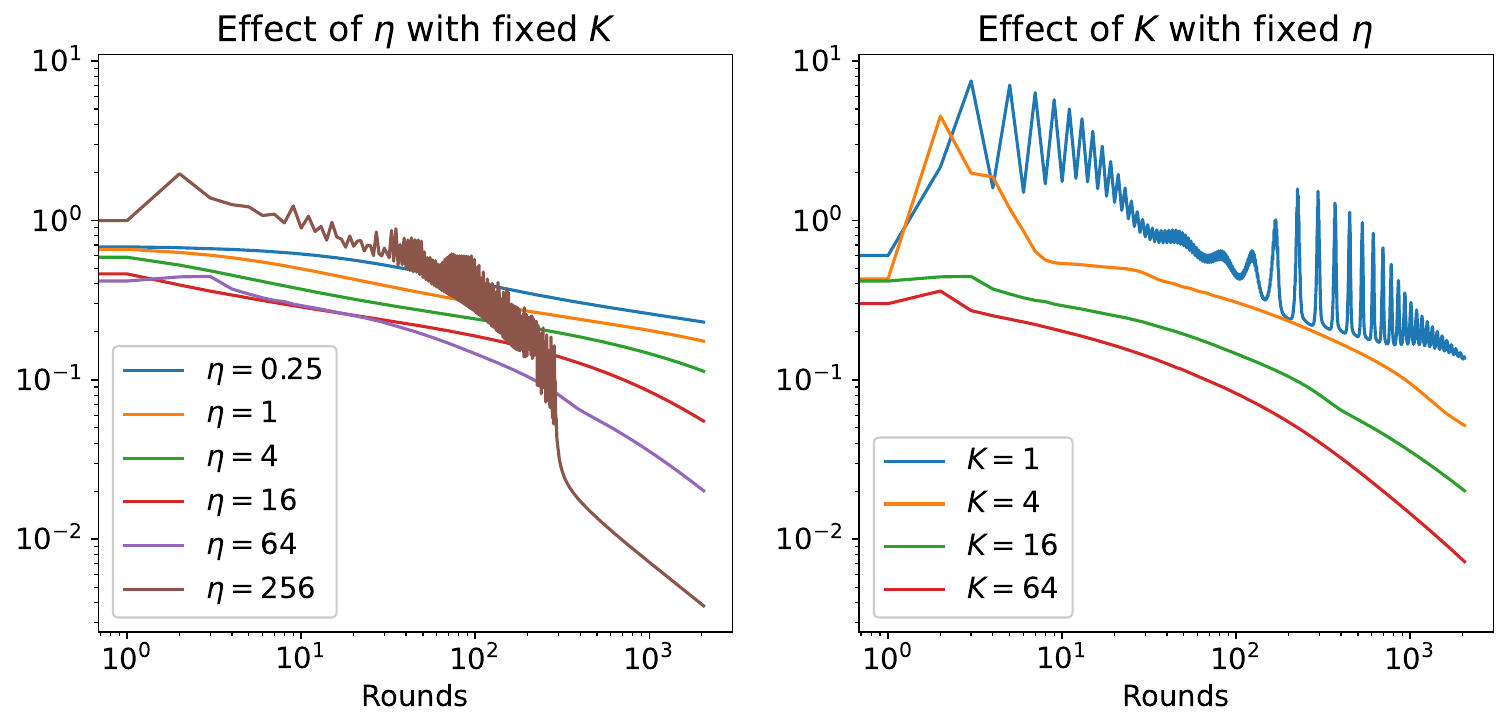}
    \caption{MNIST data. Left: $K = 16$. Right: $\eta = 64$.}
\end{subfigure}
\caption{Objective gap when varying one of $\eta, K$ and keeping the other fixed. In
general, Local GD converges faster when $\eta$ and $K$ are larger, despite the initial
instability in early rounds.}
\label{fig:train_comparison}
\end{center}
\vskip -0.2in
\end{figure*}

\begin{figure}[ht]
\vskip 0.2in
\begin{center}
\centerline{
\includegraphics[width=0.49\columnwidth]{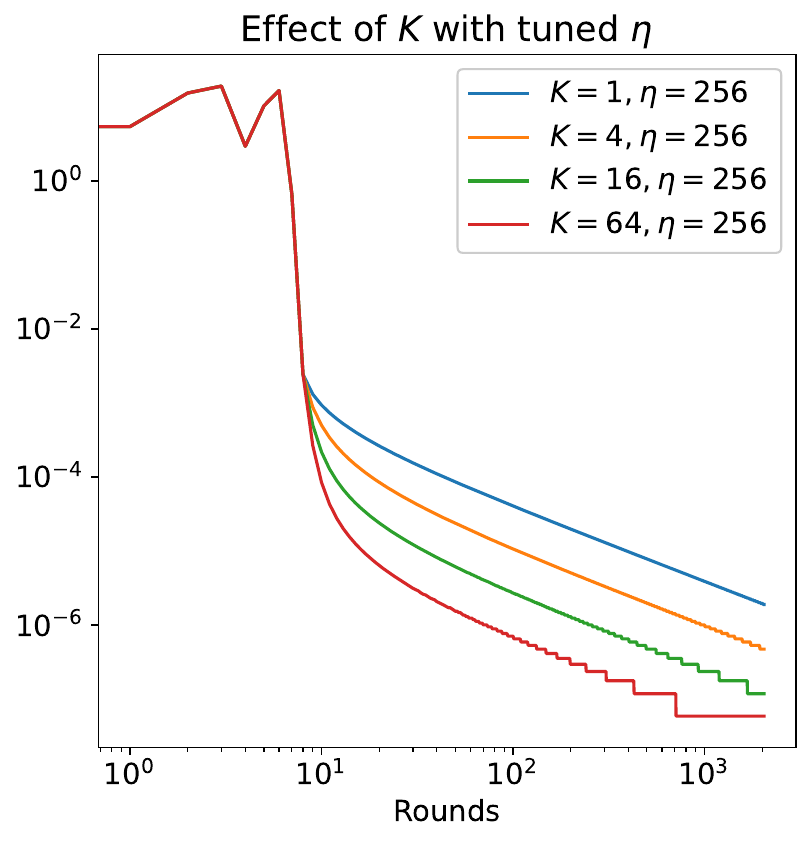}
\includegraphics[width=0.49\columnwidth]{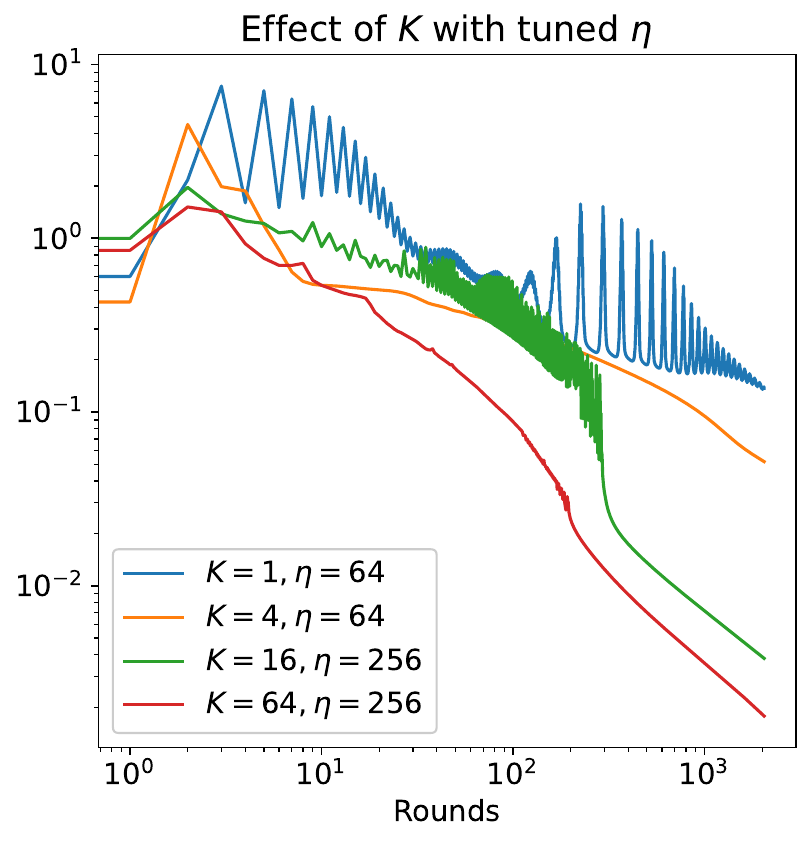}
}
\caption{Objective gap for different values of $K$ with tuned $\eta$. Left: Synthetic
data. Right: MNIST data.}
\label{fig:tune_eta}
\end{center}
\vskip -0.2in
\end{figure}

\begin{figure}[ht]
\vskip 0.2in
\begin{center}
\centerline{
\includegraphics[width=0.49\columnwidth]{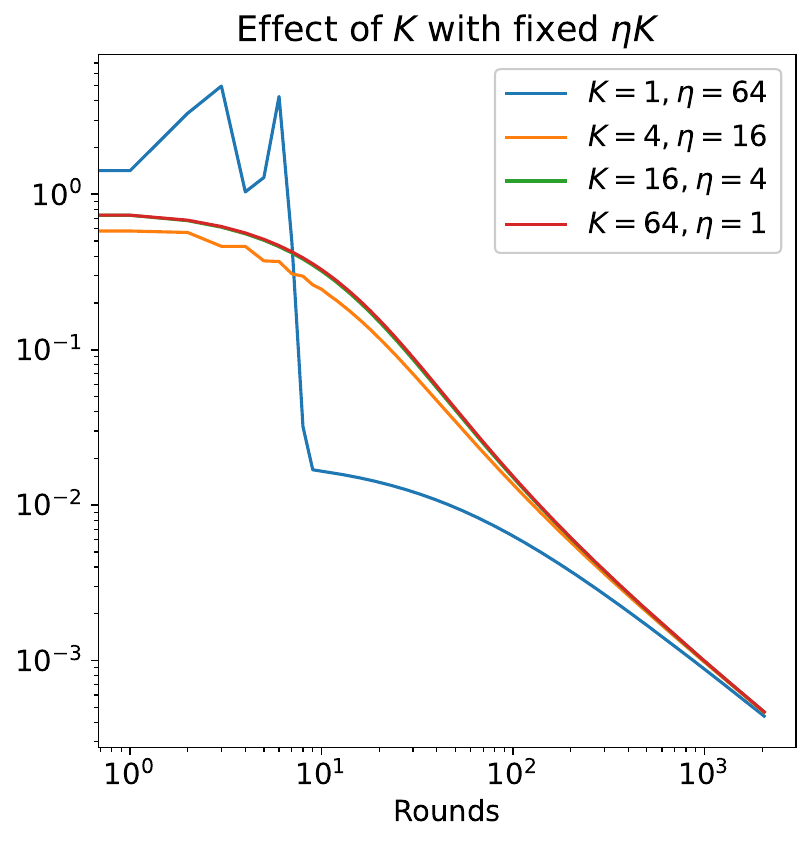}
\includegraphics[width=0.49\columnwidth]{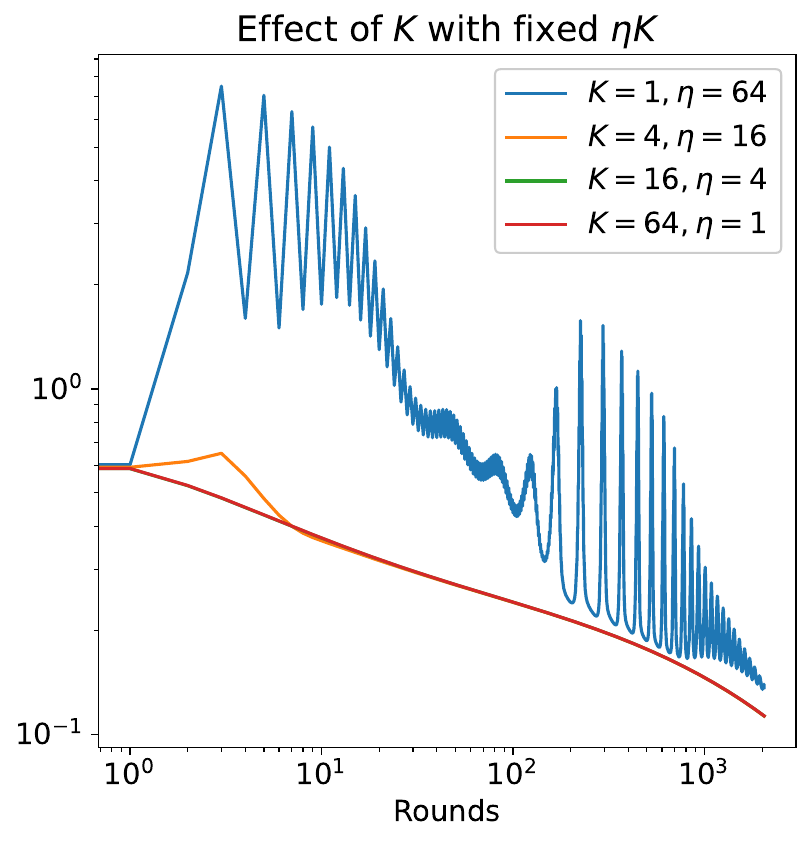}
}
\caption{Objective gap for different values of $\eta, K$ with constant $\eta K$. Left:
Synthetic data. Right: MNIST data.}
\label{fig:constant_eta_K}
\end{center}
\vskip -0.2in
\end{figure}

\section{Experiments} \label{sec:experiments}
We further investigate the behavior of Local GD for logistic regression through
experiments, in order to answer the following questions: \textbf{Q1:} Can Local GD
converge faster by choosing $\eta$ and $K$ large enough to create non-monotonic
objective decrease? \textbf{Q2:} Do local steps yield faster convergence if we tune
$\eta$ after choosing $K$? \textbf{Q3:}  Do local steps yield faster convergence if we
keep $\eta K$ constant? We investigate Q1 to empirically verify our theoretical
findings, whereas Q2 and Q3 are meant to probe the limitations of our theory: our
guarantee (Corollary \ref{cor:final_error}) does not show any benefit of local steps,
and we ask whether such a benefit occurs in practice. We further discuss this limitation
of our theory in Section \ref{sec:discussion}. Lastly, we provide an additional
experiment with synthetic data in Appendix \ref{app:hetero_exp} to evaluate how
optimization behavior is affected by heterogeneity among the margins of each client's
local dataset.

\paragraph{Setup} We evaluate Local GD for a synthetic dataset used by
\citet{crawshaw2025local} and for a subset of the MNIST dataset with binarized labels,
following \citep{wu2024large} and \citep{crawshaw2025local}. The synthetic dataset is a
simple testbed with $M=2$ clients and $n=1$ data point per client. For MNIST, we use a
common protocol \citep{karimireddy2020scaffold, crawshaw2025local} to partition 1000
MNIST images among $M=5$ clients with $n=200$ images each, in a way that induces
heterogeneous feature distributions among clients. Note that $H \leq 1/4$ for these
datasets. See Appendix \ref{app:experiment_details} for complete details of each
dataset. Additionally, we provide results with the CIFAR-10 dataset in Appendix
\ref{app:cifar_exp}. 

We run Local GD with a wide range of values for the parameters: $\eta \in \{2^{-2}, 2^0,
2^2, 2^4, 2^6, 2^8, 2^{10}\}$ and $K \in \{2^0, 2^2, 2^4, 2^6\}$. Note that the
traditional choice of $\eta = 1/H = 2^2$ is in the middle of the search range for
$\eta$, so a large number of these experiments fall outside of the scope of conventional
theory. All experiments have a communication budget of $R = 2048$ rounds.

\paragraph{Results} Our investigations of Q1, Q2, and Q3 are shown in Figures
\ref{fig:train_comparison}, \ref{fig:tune_eta}, and \ref{fig:constant_eta_K}. Note that
the results for $\eta = 2^{10}$ are not shown because all such trajectories diverged.

The loss curves in Figure \ref{fig:train_comparison} show that the final error reached
by Local GD is always made smaller when either $\eta$ or $K$ is increased while the
other is held fixed, even when such changes create instability. This answers Q1
affirmatively and is consistent with our theory. Unsurprisingly, increases to $\eta$
create higher loss spikes and require more communication rounds to reach stability,
which aligns with our theory. More surprising is that increases to $K$ actually preserve
or decrease the rounds required to reach stability while also leading to a smaller final
loss! This is consistent across both datasets and is stronger than predicted by our
theory, since the transition time $\tau$ in Theorem \ref{thm:stage2} is proportional to
$\eta K$.

Figure \ref{fig:tune_eta} shows that a larger communication interval $K$ can accelerate
convergence when $\eta$ is tuned to $K$, which answers Q2 positively. For the synthetic
data, larger choices of $K$ do not increase the time to reach stability, but they lead
to a smaller final error. In the MNIST case, we see another stabilizing effect of $K$:
larger choices of $K$ permit larger choices of $\eta$! Indeed, setting $\eta = 256$ when
$K=1$ or $K=4$ caused divergence, whereas this choice led to fast (albeit unstable)
convergence when $K=16$ or $K=64$.

Lastly, since our Theorem \ref{thm:stage2} does not distinguish the error of Local GD
when $\eta K$ is constant, Figure \ref{fig:constant_eta_K} evaluates different parameter
choices which have a common value of $\eta K$. For both datasets, the final error
reached by Local GD is nearly identical for all parameter choices, which leans toward a
negative answer for Q3. However, we can see that the number of rounds required to reach
the stable phase tends to decrease as $K$ increases, which still suggests that there may
be room for improvement in our bound of the transition time in Theorem \ref{thm:stage2}.

Together, our experimental results confirm that instability is an important ingredient
for the fast convergence of Local GD for logistic regression. Further, they suggest that
Local GD with $K > 1$ may be able to outperform GD under the same communication budget,
which is even stronger than our current guarantees. We discuss this possibility as a
direction of future research in Section \ref{sec:discussion}.

\section{Discussion} \label{sec:discussion}
We have presented the first results showing that Local GD for logistic regression can
converge with any step size $\eta > 0$ and any communication interval $K$, and our
convergence rate improves upon that guaranteed by the worst-case analysis which is known
to be tight \citep{koloskova2020unified, woodworth2020minibatch, patel2024limits}. Below
we discuss the problem-specific approach, limitations of our results, and suggest
directions for follow up work.

\paragraph{Choice of Problem Class} The conventional optimization analysis of
distributed learning focuses on providing guarantees of efficiency in the worst-case
over large classes of optimization problems. The question is, which class of problems
should we analyze? Certain classes of problems lend themselves well to theoretical
analysis, such as those satisfying a heterogeneity condition like uniformly bounded
gradient dissimilarity \citep{woodworth2020minibatch}, or bounded gradient dissimilarity
at the optimum \citep{koloskova2020unified}; however, such conditions have come into
question, since they lead to worst-case complexities that do not explain algorithm
behavior for practical problems \citep{wang2022unreasonable, patel2023on,
patel2024limits}. These works have attempted to find the ``right" heterogeneity
condition, but so far (to the best of our knowledge), no such condition has explained
the significant advantage enjoyed by Local SGD over Minibatch SGD in practice. In this
work, by focusing on a specific problem, we investigate the possibility that algorithm
performance can be explained according to the specific problem structure rather than
general heterogeneity conditions, as discussed by \citet{patel2024limits} and
\citet{crawshaw2025local}. Even though this approach is less general than the
conventional style, we believe that a narrow analysis which accurately describes
practice has a different kind of value than a general analysis which does not, and is an
important direction for the community to pursue.

\paragraph{Usefulness of Local Steps} The main limitation of our results is that our
error bound for Local GD is strictly worse than that of GD for $R$ steps
\citep{wu2024large} in terms of $M$ and $1/\gamma$ (see Table \ref{tab:error}). If we
are to accept these results, one should set $K=1$ and parallelize GD over $M$ machines
rather than use Local GD with $K > 1$, but it remains open whether our analysis for
Local GD can be improved to match (or even dominate) GD. Based on our experiments, we
conjecture that Local GD with $K > 1$ can converge faster than GD, and this suggests two
open problems: (1) Provide a lower bound of GD for logistic regression, and (2)
Determine whether Local GD with $K > 1$ can converge with error smaller than
$R^{-\alpha}$ for some $\alpha > 2$. Our current results are insufficient to show any
advantage to setting $K > 1$, not only because of the unfavorable comparison with GD,
but also because $\eta$ and $K$ appear in our Theorem \ref{thm:stage2} only through the
product $\eta K$ (excluding non-dominating terms of the transition time $\tau$). This
means that any error guaranteed by choosing stepsize $\eta$ and communication interval
$K$ can also be guaranteed with stepsize $\eta K$ and communication interval $1$, so
that an interval larger than $1$ does not produce any advantage. The challenge of
proving an advantage from local steps is fundamental in distributed optimization
\cite{woodworth2020minibatch, glasgow2022sharp, patel2024limits}, and we hope to address
this in future work.

\paragraph{Future Extensions} There are several natural extensions of our work, given
the narrow focus of the problem setting. Since SGD for logistic regression was analyzed
by \citet{wu2024large} using similar techniques as we have leveraged in this work, one
direction is to extend our analysis for Local SGD. These same techniques were applied by
\citet{cai2024large} to analyze GD for training two-layer neural networks with
approximately homogeneous activations, so another direction is to analyze the
distributed training of two-layer networks with Local GD. Lastly, one could attempt to
generalize our analysis for a larger class of problems, by formulating some general
problem class for which Local GD outperforms the existing worst-case lower bounds. We
leave these directions for future work.

\section*{Acknowledgements}
Thank you to the anonymous reviewers for the valuable feedback. Michael Crawshaw
is supported by the Institute for Digital Innovation fellowship. Mingrui Liu is supported by a ORIEI seed funding, an IDIA P3 fellowship from
George Mason University, and NSF awards \#2436217, \#2425687.

\section*{Impact Statement}
This paper presents work whose goal is to advance the field of Machine Learning. There
are many potential societal consequences of our work, none which we feel must be
specifically highlighted here.

\bibliography{references}

\begin{thebibliography}{38}
\providecommand{\natexlab}[1]{#1}
\providecommand{\url}[1]{\texttt{#1}}
\expandafter\ifx\csname urlstyle\endcsname\relax
  \providecommand{\doi}[1]{doi: #1}\else
  \providecommand{\doi}{doi: \begingroup \urlstyle{rm}\Url}\fi

\bibitem[Arjevani \& Shamir(2015)Arjevani and
  Shamir]{arjevani2015communication}
Arjevani, Y. and Shamir, O.
\newblock Communication complexity of distributed convex learning and
  optimization.
\newblock \emph{Advances in neural information processing systems}, 28, 2015.

\bibitem[Balcan et~al.(2012)Balcan, Blum, Fine, and
  Mansour]{balcan2012distributed}
Balcan, M.~F., Blum, A., Fine, S., and Mansour, Y.
\newblock Distributed learning, communication complexity and privacy.
\newblock In \emph{Conference on Learning Theory}, pp.\  26--1. JMLR Workshop
  and Conference Proceedings, 2012.

\bibitem[Cai et~al.(2024)Cai, Wu, Mei, Lindsey, and Bartlett]{cai2024large}
Cai, Y., Wu, J., Mei, S., Lindsey, M., and Bartlett, P.~L.
\newblock Large stepsize gradient descent for non-homogeneous two-layer
  networks: Margin improvement and fast optimization.
\newblock \emph{arXiv preprint arXiv:2406.08654}, 2024.

\bibitem[Cohen et~al.(2021)Cohen, Kaur, Li, Kolter, and
  Talwalkar]{cohen2021gradient}
Cohen, J., Kaur, S., Li, Y., Kolter, J.~Z., and Talwalkar, A.
\newblock Gradient descent on neural networks typically occurs at the edge of
  stability.
\newblock In \emph{International Conference on Learning Representations}, 2021.

\bibitem[Crawshaw et~al.(2025)Crawshaw, Woodworth, and Liu]{crawshaw2025local}
Crawshaw, M., Woodworth, B., and Liu, M.
\newblock Local steps speed up local gd for heterogeneous distributed logistic
  regression.
\newblock \emph{International Conference on Learning Representations}, 2025.

\bibitem[Damian et~al.(2023)Damian, Nichani, and
  Lee]{damian2023selfstabilization}
Damian, A., Nichani, E., and Lee, J.~D.
\newblock Self-stabilization: The implicit bias of gradient descent at the edge
  of stability.
\newblock In \emph{The Eleventh International Conference on Learning
  Representations}, 2023.
\newblock URL \url{https://openreview.net/forum?id=nhKHA59gXz}.

\bibitem[Dekel et~al.(2012)Dekel, Gilad-Bachrach, Shamir, and
  Xiao]{dekel2012optimal}
Dekel, O., Gilad-Bachrach, R., Shamir, O., and Xiao, L.
\newblock Optimal distributed online prediction using mini-batches.
\newblock \emph{Journal of Machine Learning Research}, 13\penalty0 (1), 2012.

\bibitem[Glasgow et~al.(2022)Glasgow, Yuan, and Ma]{glasgow2022sharp}
Glasgow, M.~R., Yuan, H., and Ma, T.
\newblock Sharp bounds for federated averaging (local sgd) and continuous
  perspective.
\newblock In \emph{International Conference on Artificial Intelligence and
  Statistics}, pp.\  9050--9090. PMLR, 2022.

\bibitem[Gunasekar et~al.(2018)Gunasekar, Lee, Soudry, and
  Srebro]{gunasekar2018characterizing}
Gunasekar, S., Lee, J., Soudry, D., and Srebro, N.
\newblock Characterizing implicit bias in terms of optimization geometry.
\newblock In \emph{International Conference on Machine Learning}, pp.\
  1832--1841. PMLR, 2018.

\bibitem[Haddadpour \& Mahdavi(2019)Haddadpour and
  Mahdavi]{haddadpour2019convergence}
Haddadpour, F. and Mahdavi, M.
\newblock On the convergence of local descent methods in federated learning.
\newblock \emph{arXiv preprint arXiv:1910.14425}, 2019.

\bibitem[Jastrzebski et~al.(2020)Jastrzebski, Szymczak, Fort, Arpit, Tabor,
  Cho*, and Geras*]{jastrzebski2020the}
Jastrzebski, S., Szymczak, M., Fort, S., Arpit, D., Tabor, J., Cho*, K., and
  Geras*, K.
\newblock The break-even point on optimization trajectories of deep neural
  networks.
\newblock In \emph{International Conference on Learning Representations}, 2020.
\newblock URL \url{https://openreview.net/forum?id=r1g87C4KwB}.

\bibitem[Ji \& Telgarsky(2019)Ji and Telgarsky]{ji2019implicit}
Ji, Z. and Telgarsky, M.
\newblock The implicit bias of gradient descent on nonseparable data.
\newblock In Beygelzimer, A. and Hsu, D. (eds.), \emph{Proceedings of the
  Thirty-Second Conference on Learning Theory}, volume~99 of \emph{Proceedings
  of Machine Learning Research}, pp.\  1772--1798. PMLR, 25--28 Jun 2019.
\newblock URL \url{https://proceedings.mlr.press/v99/ji19a.html}.

\bibitem[Ji et~al.(2021)Ji, Srebro, and Telgarsky]{ji2021fast}
Ji, Z., Srebro, N., and Telgarsky, M.
\newblock Fast margin maximization via dual acceleration.
\newblock In \emph{International Conference on Machine Learning}, pp.\
  4860--4869. PMLR, 2021.

\bibitem[Kairouz et~al.(2021)Kairouz, McMahan, Avent, Bellet, Bennis, Bhagoji,
  Bonawitz, Charles, Cormode, Cummings, et~al.]{kairouz2021advances}
Kairouz, P., McMahan, H.~B., Avent, B., Bellet, A., Bennis, M., Bhagoji, A.~N.,
  Bonawitz, K., Charles, Z., Cormode, G., Cummings, R., et~al.
\newblock Advances and open problems in federated learning.
\newblock \emph{Foundations and trends{\textregistered} in machine learning},
  14\penalty0 (1--2):\penalty0 1--210, 2021.

\bibitem[Karimireddy et~al.(2020)Karimireddy, Kale, Mohri, Reddi, Stich, and
  Suresh]{karimireddy2020scaffold}
Karimireddy, S.~P., Kale, S., Mohri, M., Reddi, S., Stich, S., and Suresh,
  A.~T.
\newblock Scaffold: Stochastic controlled averaging for federated learning.
\newblock In \emph{International conference on machine learning}, pp.\
  5132--5143. PMLR, 2020.

\bibitem[Khaled et~al.(2020)Khaled, Mishchenko, and
  Richt{\'a}rik]{khaled2020tighter}
Khaled, A., Mishchenko, K., and Richt{\'a}rik, P.
\newblock Tighter theory for local sgd on identical and heterogeneous data.
\newblock In \emph{International Conference on Artificial Intelligence and
  Statistics}, pp.\  4519--4529. PMLR, 2020.

\bibitem[Koloskova et~al.(2020)Koloskova, Loizou, Boreiri, Jaggi, and
  Stich]{koloskova2020unified}
Koloskova, A., Loizou, N., Boreiri, S., Jaggi, M., and Stich, S.
\newblock A unified theory of decentralized sgd with changing topology and
  local updates.
\newblock In \emph{International Conference on Machine Learning}, pp.\
  5381--5393. PMLR, 2020.

\bibitem[Mcdonald et~al.(2009)Mcdonald, Mohri, Silberman, Walker, and
  Mann]{mcdonald2009efficient}
Mcdonald, R., Mohri, M., Silberman, N., Walker, D., and Mann, G.
\newblock Efficient large-scale distributed training of conditional maximum
  entropy models.
\newblock \emph{Advances in neural information processing systems}, 22, 2009.

\bibitem[McDonald et~al.(2010)McDonald, Hall, and
  Mann]{mcdonald2010distributed}
McDonald, R., Hall, K., and Mann, G.
\newblock Distributed training strategies for the structured perceptron.
\newblock In \emph{Human Language Technologies: The 2010 Annual Conference of
  the North American Chapter of the Association for Computational Linguistics},
  pp.\  456--464. Association for Computational Linguistics, 2010.

\bibitem[McMahan et~al.(2017)McMahan, Moore, Ramage, Hampson, and
  Arcas]{mcmahan2017communication}
McMahan, B., Moore, E., Ramage, D., Hampson, S., and Arcas, B. A.~y.
\newblock {Communication-Efficient Learning of Deep Networks from Decentralized
  Data}.
\newblock In Singh, A. and Zhu, J. (eds.), \emph{Proceedings of the 20th
  International Conference on Artificial Intelligence and Statistics},
  volume~54 of \emph{Proceedings of Machine Learning Research}, pp.\
  1273--1282. PMLR, 20--22 Apr 2017.
\newblock URL \url{https://proceedings.mlr.press/v54/mcmahan17a.html}.

\bibitem[Nacson et~al.(2019)Nacson, Srebro, and Soudry]{nacson2019stochastic}
Nacson, M.~S., Srebro, N., and Soudry, D.
\newblock Stochastic gradient descent on separable data: Exact convergence with
  a fixed learning rate.
\newblock In \emph{The 22nd International Conference on Artificial Intelligence
  and Statistics}, pp.\  3051--3059. PMLR, 2019.

\bibitem[Patel et~al.(2023)Patel, Glasgow, Wang, Joshi, and
  Srebro]{patel2023on}
Patel, K.~K., Glasgow, M., Wang, L., Joshi, N., and Srebro, N.
\newblock On the still unreasonable effectiveness of federated averaging for
  heterogeneous distributed learning.
\newblock In \emph{Federated Learning and Analytics in Practice: Algorithms,
  Systems, Applications, and Opportunities}, 2023.
\newblock URL \url{https://openreview.net/forum?id=vhS68bKv7x}.

\bibitem[Patel et~al.(2024)Patel, Glasgow, Zindari, Wang, Stich, Cheng, Joshi,
  and Srebro]{patel2024limits}
Patel, K.~K., Glasgow, M., Zindari, A., Wang, L., Stich, S.~U., Cheng, Z.,
  Joshi, N., and Srebro, N.
\newblock The limits and potentials of local sgd for distributed heterogeneous
  learning with intermittent communication.
\newblock In Agrawal, S. and Roth, A. (eds.), \emph{Proceedings of Thirty
  Seventh Conference on Learning Theory}, volume 247 of \emph{Proceedings of
  Machine Learning Research}, pp.\  4115--4157. PMLR, 30 Jun--03 Jul 2024.
\newblock URL \url{https://proceedings.mlr.press/v247/patel24a.html}.

\bibitem[Reddi et~al.(2021)Reddi, Charles, Zaheer, Garrett, Rush,
  Kone{\v{c}}n{\'y}, Kumar, and McMahan]{reddi2021adaptive}
Reddi, S.~J., Charles, Z., Zaheer, M., Garrett, Z., Rush, K.,
  Kone{\v{c}}n{\'y}, J., Kumar, S., and McMahan, H.~B.
\newblock Adaptive federated optimization.
\newblock In \emph{International Conference on Learning Representations}, 2021.
\newblock URL \url{https://openreview.net/forum?id=LkFG3lB13U5}.

\bibitem[Shamir \& Srebro(2014)Shamir and Srebro]{shamir2014distributed}
Shamir, O. and Srebro, N.
\newblock Distributed stochastic optimization and learning.
\newblock In \emph{2014 52nd Annual Allerton Conference on Communication,
  Control, and Computing (Allerton)}, pp.\  850--857. IEEE, 2014.

\bibitem[Soudry et~al.(2018)Soudry, Hoffer, Nacson, Gunasekar, and
  Srebro]{soudry2018implicit}
Soudry, D., Hoffer, E., Nacson, M.~S., Gunasekar, S., and Srebro, N.
\newblock The implicit bias of gradient descent on separable data.
\newblock \emph{Journal of Machine Learning Research}, 19\penalty0
  (70):\penalty0 1--57, 2018.
\newblock URL \url{http://jmlr.org/papers/v19/18-188.html}.

\bibitem[Stich(2019)]{stich2019local}
Stich, S.~U.
\newblock Local sgd converges fast and communicates little.
\newblock In \emph{ICLR 2019-International Conference on Learning
  Representations}, 2019.

\bibitem[Tang et~al.(2020)Tang, Shi, Wang, Li, and Chu]{tang2020communication}
Tang, Z., Shi, S., Wang, W., Li, B., and Chu, X.
\newblock Communication-efficient distributed deep learning: A comprehensive
  survey.
\newblock \emph{arXiv preprint arXiv:2003.06307}, 2020.

\bibitem[Verbraeken et~al.(2020)Verbraeken, Wolting, Katzy, Kloppenburg,
  Verbelen, and Rellermeyer]{verbraeken2020survey}
Verbraeken, J., Wolting, M., Katzy, J., Kloppenburg, J., Verbelen, T., and
  Rellermeyer, J.~S.
\newblock A survey on distributed machine learning.
\newblock \emph{Acm computing surveys (csur)}, 53\penalty0 (2):\penalty0 1--33,
  2020.

\bibitem[Wang et~al.(2021)Wang, Charles, Xu, Joshi, McMahan, Al-Shedivat,
  Andrew, Avestimehr, Daly, Data, et~al.]{wang2021field}
Wang, J., Charles, Z., Xu, Z., Joshi, G., McMahan, H.~B., Al-Shedivat, M.,
  Andrew, G., Avestimehr, S., Daly, K., Data, D., et~al.
\newblock A field guide to federated optimization.
\newblock \emph{arXiv preprint arXiv:2107.06917}, 2021.

\bibitem[Wang et~al.(2022)Wang, Das, Joshi, Kale, Xu, and
  Zhang]{wang2022unreasonable}
Wang, J., Das, R., Joshi, G., Kale, S., Xu, Z., and Zhang, T.
\newblock On the unreasonable effectiveness of federated averaging with
  heterogeneous data.
\newblock \emph{arXiv preprint arXiv:2206.04723}, 2022.

\bibitem[Woodworth et~al.(2020{\natexlab{a}})Woodworth, Patel, Stich, Dai,
  Bullins, Mcmahan, Shamir, and Srebro]{woodworth2020local}
Woodworth, B., Patel, K.~K., Stich, S., Dai, Z., Bullins, B., Mcmahan, B.,
  Shamir, O., and Srebro, N.
\newblock Is local sgd better than minibatch sgd?
\newblock In \emph{International Conference on Machine Learning}, pp.\
  10334--10343. PMLR, 2020{\natexlab{a}}.

\bibitem[Woodworth et~al.(2020{\natexlab{b}})Woodworth, Patel, and
  Srebro]{woodworth2020minibatch}
Woodworth, B.~E., Patel, K.~K., and Srebro, N.
\newblock Minibatch vs local sgd for heterogeneous distributed learning.
\newblock \emph{Advances in Neural Information Processing Systems},
  33:\penalty0 6281--6292, 2020{\natexlab{b}}.

\bibitem[Wu et~al.(2024{\natexlab{a}})Wu, Bartlett, Telgarsky, and
  Yu]{wu2024large}
Wu, J., Bartlett, P.~L., Telgarsky, M., and Yu, B.
\newblock Large stepsize gradient descent for logistic loss: Non-monotonicity
  of the loss improves optimization efficiency.
\newblock In Agrawal, S. and Roth, A. (eds.), \emph{Proceedings of Thirty
  Seventh Conference on Learning Theory}, volume 247 of \emph{Proceedings of
  Machine Learning Research}, pp.\  5019--5073. PMLR, 30 Jun--03 Jul
  2024{\natexlab{a}}.
\newblock URL \url{https://proceedings.mlr.press/v247/wu24b.html}.

\bibitem[Wu et~al.(2024{\natexlab{b}})Wu, Braverman, and Lee]{wu2024implicit}
Wu, J., Braverman, V., and Lee, J.~D.
\newblock Implicit bias of gradient descent for logistic regression at the edge
  of stability.
\newblock \emph{Advances in Neural Information Processing Systems}, 36,
  2024{\natexlab{b}}.

\bibitem[Xu et~al.(2023)Xu, Zhang, Andrew, Choquette, Kairouz, Mcmahan,
  Rosenstock, and Zhang]{xu2023federated}
Xu, Z., Zhang, Y., Andrew, G., Choquette, C., Kairouz, P., Mcmahan, B.,
  Rosenstock, J., and Zhang, Y.
\newblock Federated learning of gboard language models with differential
  privacy.
\newblock In \emph{Proceedings of the 61st Annual Meeting of the Association
  for Computational Linguistics (Volume 5: Industry Track)}, pp.\  629--639,
  2023.

\bibitem[Zhang et~al.(2013)Zhang, Duchi, Jordan, and
  Wainwright]{zhang2013information}
Zhang, Y., Duchi, J., Jordan, M.~I., and Wainwright, M.~J.
\newblock Information-theoretic lower bounds for distributed statistical
  estimation with communication constraints.
\newblock \emph{Advances in Neural Information Processing Systems}, 26, 2013.

\bibitem[Zinkevich et~al.(2010)Zinkevich, Weimer, Li, and
  Smola]{zinkevich2010parallelized}
Zinkevich, M., Weimer, M., Li, L., and Smola, A.~J.
\newblock Parallelized stochastic gradient descent.
\newblock In \emph{Advances in neural information processing systems}, pp.\
  2595--2603, 2010.

\end{thebibliography}
\bibliographystyle{icml2025}

%%%%%%%%%%%%%%%%%%%%%%%%%%%%%%%%%%%%%%%%%%%%%%%%%%%%%%%%%%%%%%%%%%%%%%%%%%%%%%%
%%%%%%%%%%%%%%%%%%%%%%%%%%%%%%%%%%%%%%%%%%%%%%%%%%%%%%%%%%%%%%%%%%%%%%%%%%%%%%%
% APPENDIX
%%%%%%%%%%%%%%%%%%%%%%%%%%%%%%%%%%%%%%%%%%%%%%%%%%%%%%%%%%%%%%%%%%%%%%%%%%%%%%%
%%%%%%%%%%%%%%%%%%%%%%%%%%%%%%%%%%%%%%%%%%%%%%%%%%%%%%%%%%%%%%%%%%%%%%%%%%%%%%%
\newpage
\appendix
\onecolumn

\section{Proofs of Main Results} \label{app:proofs}

\subsection{Proof of Theorem \ref{thm:stage1}} \label{app:unstable_proofs}

\begin{lemma}[Restatement of Lemma \ref{lem:param_ub}] \label{lem:app_param_ub}
For every $r \geq 0$,
\begin{equation}
    \|\vw_r\| \leq \|\vw_0\| + \frac{\sqrt{2} + \eta + \log(1 + \eta \gamma^2 K r^2)}{\gamma}.
\end{equation}
\end{lemma}

\begin{proof}
Recall that \citet{wu2024large} introduced a large stepsize analysis of GD for logistic
regression, which provides an upper bound on the norm of the parameter at each step. We
wish to achieve a similar bound for the norm of the parameter found by Local GD. To
accomplish this, we treat the local training of each client during each round as GD on a
logistic regression problem, and we apply the "split comparator" technique of
\citet{wu2024large}. This leads to a recursive upper bound on the norm of the global
parameter $\|\vw_r\|$, and unrolling yields the desired bound. We demonstrate this
argument below.

Let $0 \leq s < r$ and $m \in [M]$. Define $\vu_1 = \lambda_1 \vw_*, \vu_2 = \lambda_2
\vw_*$, and $\vu = \vu_1 + \vu_2$, where $\lambda_1, \lambda_2$ will be chosen later and
will not depend on $s$ or $m$. Note that $\vu$ is a scalar multiple of $\vw_*$, which is
the maximum margin predictor of the global dataset, not that of any local dataset. We
start by applying the split comparator technique of \cite{wu2024large} to the local
updates of client $m$ at round $s$, which takes $K$ gradient steps with learning rate
$\eta$ on the objective $F_m$, initialized from $\vw_s$. For every $0 \leq k < K$,
\begin{align}
    \|\vw_{s,k+1}^m - \vu\|^2 &= \|(\vw_{s,k}^m - \vu) + (\vw_{s,k+1}^m - \vw_{s,k}^m)\|^2 \\
    &= \|\vw_{s,k}^m - \vu\|^2 + 2 \left\langle \vw_{s,k+1}^m - \vw_{s,k}^m, \vw_{s,k}^m - \vu \right\rangle + \|\vw_{s,k+1}^m - \vw_{s,k}^m\|^2 \\
    &= \|\vw_{s,k}^m - \vu\|^2 + 2 \eta \left\langle \nabla F_m(\vw_{s,k}^m), \vu - \vw_{s,k}^m \right\rangle + \eta^2 \|\nabla F_m(\vw_{s,k}^m)\|^2 \\
    &= \|\vw_{s,k}^m - \vu\|^2 + \underbrace{2 \eta \left\langle \nabla F_m(\vw_{s,k}^m), \vu_1 - \vw_{s,k}^m \right\rangle}_{A_1} \\
    &\quad + \underbrace{2 \eta \left\langle \nabla F_m(\vw_{s,k}^m), \vu_2 \right\rangle + \eta^2 \|\nabla F_m(\vw_{s,k}^m)\|^2}_{A_2} \label{eq:param_ub_inter}
\end{align}
The first term $A_1$ is easily bounded by convexity of $F_m$:
\begin{equation}
    A_1 = 2 \eta \left\langle \nabla F_m(\vw_{s,k}^m), \vu_1 - \vw_{s,k}^m \right\rangle \leq 2 \eta (F_m(\vu_1) - F_m(\vw_{s,k}^m)).
\end{equation}
The second term $A_2$ can be bounded by the Lipschitz property of $F_m$ together with a choice of $\vu_2$:
\begin{align}
    A_2 &= \eta \left( 2 \left\langle \nabla F_m(\vw_{s,k}^m), \vu_2 \right\rangle + \eta \|\nabla F_m(\vw_{s,k}^m)\|^2 \right) \\
    &\Eqmark{i}{=} \eta \left( -\frac{2}{n} \sum_{i=1}^n \frac{\left \langle \vx_i^m, \vu_2 \right\rangle}{1 + \exp(\langle \vw_{s,k}^m, \vx_i^m \rangle)} + \eta \left\| \frac{1}{n} \sum_{i=1}^n \frac{\vx_i^m}{1 + \exp(\langle \vw_{s,k}^m, \vx_i^m \rangle)} \right\|^2 \right) \\
    &\Eqmark{ii}{\leq} \eta \left( -\frac{2 \lambda_2}{n} \sum_{i=1}^n \frac{\left \langle \vx_i^m, \vw_* \right\rangle}{1 + \exp(\langle \vw_{s,k}^m, \vx_i^m \rangle)} + \frac{\eta}{n} \sum_{i=1}^n \left\| \frac{\vx_i^m}{1 + \exp(\langle \vw_{s,k}^m, \vx_i^m \rangle)} \right\|^2 \right) \\
    &\Eqmark{iii}{\leq} \eta \left( -\frac{2 \gamma \lambda_2}{n} \sum_{i=1}^n \frac{1}{1 + \exp(\langle \vw_{s,k}^m, \vx_i^m \rangle)} + \frac{\eta}{n} \sum_{i=1}^n \left\| \frac{\vx_i^m}{1 + \exp(\langle \vw_{s,k}^m, \vx_i^m \rangle)} \right\| \right) \\
    &\Eqmark{iv}{\leq} \eta \left( -\frac{2 \gamma \lambda_2}{n} \sum_{i=1}^n \frac{1}{1 + \exp(\langle \vw_{s,k}^m, \vx_i^m \rangle)} + \frac{\eta}{n} \sum_{i=1}^n \frac{1}{1 + \exp(\langle \vw_{s,k}^m, \vx_i^m \rangle)} \right) \\
    &= \frac{\eta}{n} \sum_{i=1}^n \frac{-2 \gamma \lambda_2 + \eta}{1 + \exp(\langle \vw_{s,k}^m, \vx_i^m \rangle)},
\end{align}
where $(i)$ uses the definition of $\nabla F_m$, $(ii)$ uses the definition of $\vu_2$
and Jensen's inequality, and both $(iii)$ and $(iv)$ use $\|\vx_i^m\| \leq 1$.
Therefore, choosing $\lambda_2 = \eta/(2 \gamma)$ implies that $A_2 \leq 0$. Plugging
back to \Eqref{eq:param_ub_inter},
\begin{align}
    \|\vw_{s,k+1}^m - \vu\|^2 &\leq \|\vw_{s,k}^m - \vu\|^2 + 2 \eta (F_m(\vu_1) - F_m(\vw_{s,k}^m)) \\
    F_m(\vw_{s,k}^m) &\leq \frac{\|\vw_{s,k}^m - \vu\|^2 - \|\vw_{s,k+1}^m - \vu\|^2}{2 \eta} + F_m(\vu_1).
\end{align}
Averaging over $k \in \{0, \ldots, K-1\}$,
\begin{align}
    \frac{1}{K} \sum_{k=0}^{K-1} F_m(\vw_{s,k}^m) &\leq \frac{\|\vw_s - \vu\|^2 - \|\vw_{s,K}^m - \vu\|^2}{2 \eta K} + F_m(\vu_1) \\
    \frac{\|\vw_{s,K}^m - \vu\|^2}{2 \eta K} + \frac{1}{K} \sum_{k=0}^{K-1} F_m(\vw_{s,k}^m) &\leq \frac{\|\vw_s - \vu\|^2}{2 \eta K} + F_m(\vu_1).
\end{align}
In particular, this implies
\begin{equation}
    \frac{\|\vw_{s,K}^m - \vu\|^2}{2 \eta K} \leq \frac{\|\vw_s - \vu\|^2}{2 \eta K} + F_m(\vu_1),
\end{equation}
so
\begin{equation}
    \|\vw_{s,K}^m - \vu\| \leq \sqrt{\|\vw_s - \vu\|^2 + 2 \eta K F_m(\vu_1)} \leq \|\vw_s - \vu\| + \sqrt{2 \eta K F_m(\vu_1)}.
\end{equation}
Recall that $\vw_{s+1} = \frac{1}{M} \sum_{m=1}^M \vw_{s,K}^m$. So averaging over $m$,
\begin{align}
    \|\vw_{s+1} - \vu\| &= \left\| \frac{1}{M} \sum_{m=1}^M \vw_{s,k}^m - \vu \right\| \leq \frac{1}{M} \sum_{m=1}^M \left\| \vw_{s,k}^m - \vu \right\| \\
    &\leq \|\vw_s - \vu\| + \frac{1}{M} \sum_{m=1}^M \sqrt{2 \eta K F_m(\vu_1)} \\
    &\Eqmark{i}{\leq} \|\vw_s - \vu\| + \sqrt{2 \eta K F(\vu_1)},
\end{align}
where $(i)$ uses the fact that $\sqrt{\cdot}$ is concave together with Jensen's
inequality. We can now unroll this recursion over $s \in \{0, \ldots, r-1\}$ to obtain
\begin{equation}
    \|\vw_r - \vu\| \leq \|\vw_0 - \vu\| + \sqrt{2 \eta K r^2 F(\vu_1)} \leq \|\vw_0\| + \|\vu\| + \sqrt{2 \eta K r^2 F(\vu_1)}.
\end{equation}
so
\begin{align}
    \|\vw_r\| &\leq \|\vw_r - \vu\| + \|\vu\| \leq \|\vw_0\| + 2 \|\vu\| + \sqrt{2 \eta K r^2 F(\vu_1)} \\
    &= \|\vw_0\| + 2 \lambda_1 + 2 \lambda_2 + \sqrt{2 \eta K r^2 F(\lambda_1 \vw_*)}.
\end{align}
It only remains to choose $\lambda_1$. Note that
\begin{align}
    F(\lambda_1 \vw_*) &= \frac{1}{Mn} \sum_{m=1}^M \sum_{i=1}^n \log(1 + \exp(-\lambda_1 \langle \vw_*, \vx_i^m \rangle)) \\
    &\Eqmark{i}{\leq} \frac{1}{Mn} \sum_{m=1}^M \sum_{i=1}^n \exp(-\lambda_1 \langle \vw_*, \vx_i^m \rangle) \\
    &\Eqmark{ii}{\leq} \exp(-\lambda_1 \gamma),
\end{align}
where $(i)$ uses $\log(1 + x) \leq x$ for $x \geq 0$ and $(ii)$ uses the definition of
$\vw_*$. Therefore, choosing $\lambda_1 = \frac{1}{\gamma} \log(1 + \eta \gamma^2 K
r^2)$ yields
\begin{equation}
    F(\lambda_1 \vw_*) \leq \frac{1}{1 + \eta \gamma^2 K r^2} \leq \frac{1}{\eta \gamma^2 K r^2},
\end{equation}
so
\begin{align}
    \|\vw_r\| &\leq \|\vw_0\| + \frac{2}{\gamma} \log(1 + \eta \gamma^2 K r^2) + \frac{\eta}{\gamma} + \sqrt{2 \eta K r^2 \frac{1}{\eta \gamma^2 K r^2}} \\
    &= \|\vw_0\| + \frac{\sqrt{2} + \eta + \log(1 + \eta \gamma^2 K r^2)}{\gamma}.
\end{align}
\end{proof}

\begin{theorem}[Restatement of Theorem \ref{thm:stage1}] \label{thm:app_stage1}
For every $r \geq 0$, Local GD satisfies
\begin{equation}
    \frac{1}{r} \sum_{s=0}^{r-1} F(\vw_s) \leq 26 \frac{\|\vw_0\|^2 + 1 + \log^2(K + \eta K \gamma^2 r) + \eta^2 K^2}{\eta \gamma^4 r}.
\end{equation}
\end{theorem}

\begin{proof}
To achieve this bound on the loss of Local GD, we again adapt the split comparator
technique of \cite{wu2024large}. This time, we consider the trajectory of the global
model $\vw_r$, instead of the trajectory of locally updated models $\vw_{r,k}^m$ as in
Lemma \ref{lem:param_ub}. To apply this technique for Local GD, we have to account for
the fact that the update direction $\vw_{r+1} - \vw_r$ is not equal to the global
gradient $\nabla F(\vw_r)$. However, both the update direction and the global gradient
are linear combinations of the data $\{\vx_i^m\}_{m,i}$, and we account for the
discrepancy between the two by bounding the ratio of their linear combination
coefficients.

Let $\vu_1 = \lambda_1 \vw_*, \vu_2 = \lambda_2 \vw_*$, where $\lambda_1$ and
$\lambda_2$ will be determined later, and let $\vu = \vu_1 + \vu_2$. Then
\begin{align}
    \|\vw_{s+1} - \vu\|^2 &= \left\| (\vw_s - \vu) + (\vw_{s+1} - \vw_s) \right\|^2 \\
    &= \|\vw_s - \vu\|^2 + 2 \langle \vw_{s+1} - \vw_s, \vw_s - \vu \rangle + \|\vw_{s+1} - \vw_s\|^2 \\
    &= \|\vw_s - \vu\|^2 + \frac{2 \eta}{M} \sum_{m=1}^M \sum_{k=0}^{K-1} \left\langle \nabla F_m(\vw_{s,k}^m), \vu - \vw_s \right\rangle + \eta^2 \left\| \frac{1}{M} \sum_{m=1}^M \sum_{k=0}^{K-1} \nabla F_m(\vw_{s,k}^m) \right\|^2 \\
    &= \|\vw_s - \vu\|^2 + \underbrace{\frac{2 \eta}{M} \sum_{m=1}^M \sum_{k=0}^{K-1} \left\langle \nabla F_m(\vw_{s,k}^m), \vu_1 - \vw_s \right\rangle}_{A_1} \\
    &\quad + \underbrace{\frac{2 \eta}{M} \sum_{m=1}^M \sum_{k=0}^{K-1} \left\langle \nabla F_m(\vw_{s,k}^m), \vu_2 \right\rangle + \eta^2 \left\| \frac{1}{M} \sum_{m=1}^M \sum_{k=0}^{K-1} \nabla F_m(\vw_{s,k}^m) \right\|^2}_{A_2}. \label{eq:comparator_inter}
\end{align}

To bound $A_1$, we express the local gradient of the local models $\nabla
F_m(\vw_{s,k}^m)$ in terms of the local gradient of the preceding global model $\nabla
F_m(\vw_s)$. For any $\vw$,
\begin{equation} \label{eq:local_grad}
    \nabla F_m(\vw) = \frac{1}{n} \sum_{i=1}^n \nabla F_{m,i}(\vw) = \frac{-1}{n} \sum_{i=1}^n \frac{\vx_i^m}{1 + \exp(\langle \vx_i^m, \vw \rangle)}.
\end{equation}
So denoting $\beta_{s,i,k}^m = (1 + \exp(b_{s,i}^m))/(1 + \exp(b_{s,i,k}^m))$ and $F_{m,i}(\vw) = \log(1 + \exp(-\langle \vw, \vx_i^m \rangle))$,
\begin{equation}
    \nabla F_m(\vw_{s,k}^m) = \frac{1}{n} \sum_{i=1}^n \frac{-\vx_i^m}{1 + \exp(b_{s,i,k}^m)} = \frac{1}{n} \sum_{i=1}^n \frac{1 + \exp(b_{s,i}^m)}{1 + \exp(b_{s,i,k}^m)} \frac{-\vx_i^m}{1 + \exp(b_{s,i}^m)} = \frac{1}{n} \sum_{i=1}^n \beta_{s,i,k}^m \nabla F_{m,i}(\vw_s).
\end{equation}
Notice, from the definition of $\beta_{s,k}^m$,
\begin{equation}
    \beta_{s,k}^m := \frac{1}{K} \sum_{k=0}^{K-1} \frac{|\ell'(b_{s,i,k}^m)|}{|\ell'(b_{s,i}^m)|} = \frac{1}{K} \sum_{k=0}^{K-1} \frac{1 + \exp(b_{s,i}^m)}{1 + \exp(b_{s,i,k}^m)} = \frac{1}{K} \sum_{k=0}^{K-1} \beta_{s,i,k}^m,
\end{equation}
so
\begin{align}
    A_1 &= \frac{2 \eta}{Mn} \sum_{m=1}^M \sum_{k=0}^{K-1} \sum_{i=1}^n \beta_{s,i,k}^m \langle \nabla F_{m,i}(\vw_s), \vu_1 - \vw_s \rangle \\
    &\Eqmark{i}{\leq} \frac{2 \eta}{Mn} \sum_{m=1}^M \sum_{k=0}^{K-1} \sum_{i=1}^n \beta_{s,i,k}^m (F_{m,i}(\vu_1) - F_{m,i}(\vw_s)) \\
    &= \frac{2 \eta K}{Mn} \sum_{m=1}^M \sum_{i=1}^n \beta_{s,i}^m F_{m,i}(\vu_1) - \frac{2 \eta K}{Mn} \sum_{m=1}^M \sum_{i=1}^n \beta_{s,i}^m F_{m,i}(\vw_s). \label{eq:A1_ub_inter}
\end{align}
where $(i)$ uses the convexity of $F_{m,i}$. We can now bound the two terms of
\Eqref{eq:A1_ub_inter} with upper and lower bounds of $\beta_{s,i}^m$, respectively.
Denoting $\phi = \|\vw_0\| + \frac{\sqrt{2} + \eta + \log(1 + \eta \gamma^2 K
r^2)}{\gamma}$,
\begin{align}
    \beta_{s,i}^m &= \frac{1}{K} \sum_{k=0}^{K-1} \frac{1 + \exp(b_{s,i}^m)}{1 + \exp(b_{s,i,k}^m)} \leq 1 + \exp(b_{s,i}^m) = 1 + \exp(\langle \vw_s, \vx_i^m \rangle) \\
    &\Eqmark{i}{\leq} 1 + \exp(\|\vw_s\|) \Eqmark{ii}{\leq} 1 + \exp \left( \|\vw_0\| + \frac{\sqrt{2} + \eta + \log(1 + \eta \gamma^2 K s^2)}{\gamma} \right) \\
    &\leq 2 \exp(\phi), \label{eq:A1_beta_ub}
\end{align}
where $(i)$ uses Cauchy-Schwarz together with $\|\vx_i^m\| \leq 1$ and $(ii)$ uses
Lemma \ref{lem:param_ub}. Also,
\begin{equation} \label{eq:A1_beta_lb}
    \beta_{s,i}^m = \frac{1}{K} \sum_{k=0}^{K-1} \frac{1 + \exp(b_{s,i}^m)}{1 + \exp(b_{s,i,k}^m)} \geq \frac{1}{K} \frac{1 + \exp(b_{s,i}^m)}{1 + \exp(b_{s,i,0}^m)} = \frac{1}{K}.
\end{equation}
The step $\beta_{s,i}^m \geq \frac{1}{K}$ was mentioned in our proof overview, and it
will be used again in the proof of Lemma \ref{lem:transition_time}. See Lemma
\ref{lem:beta_ub} for a discussion on the tightness of this bound. Plugging \Eqref{eq:A1_beta_ub} and \Eqref{eq:A1_beta_lb} into \Eqref{eq:A1_ub_inter},
\begin{align}
    A_1 &\leq \frac{4 \eta K \exp(\phi)}{Mn} \sum_{m=1}^M \sum_{i=1}^n F_{m,i}(\vu_1) - \frac{2 \eta}{Mn} \sum_{m=1}^M \sum_{i=1}^n F_{m,i}(\vw_s) \\
    &\leq 4 \eta K \exp(\phi) F(\vu_1) - 2 \eta F(\vw_s). \label{eq:A1_ub}
\end{align}
This bounds $A_1$. For $A_2$,
\begin{align}
    A_2 &= \frac{2 \eta}{M} \sum_{m=1}^M \sum_{k=0}^{K-1} \left\langle \nabla F_m(\vw_{s,k}^m), \vu_2 \right\rangle + \eta^2 K^2 \left\| \frac{1}{MK} \sum_{m=1}^M \sum_{k=0}^{K-1} \nabla F_m(\vw_{s,k}^m) \right\|^2 \\
    &\leq \frac{2 \eta}{M} \sum_{m=1}^M \sum_{k=0}^{K-1} \left\langle \nabla F_m(\vw_{s,k}^m), \vu_2 \right\rangle + \frac{\eta^2 K}{M} \sum_{m=1}^M \sum_{k=0}^{K-1} \left\| \nabla F_m(\vw_{s,k}^m) \right\|^2 \\
    &= \frac{\eta}{M} \sum_{m=1}^M \sum_{k=0}^{K-1} \left( 2 \left\langle \nabla F_m(\vw_{r,k}^m), \vu_2 \right\rangle + \eta K \left\| \nabla F_m(\vw_{s,k}^m) \right\|^2 \right) \\
    &\Eqmark{i}{\leq} \frac{\eta}{M} \sum_{m=1}^M \sum_{k=0}^{K-1} \left( 2 \left\langle \nabla F_m(\vw_{s,k}^m), \vu_2 \right\rangle + \eta K \left\| \nabla F_m(\vw_{s,k}^m) \right\| \right) \\
    &= \frac{\eta}{Mn} \sum_{m=1}^M \sum_{k=0}^{K-1} \sum_{i=1}^n \left( -\frac{2 \langle \vx_i^m, \vu_2 \rangle}{1 + \exp(\langle \vx_i^m, \vw_{s,k}^m)} + \frac{\eta K \|\vx_i^m\|}{1 + \exp(\langle \vx_i^m, \vw_{s,k}^m \rangle)} \right) \\
    &\Eqmark{ii}{=} \frac{\eta}{Mn} \sum_{m=1}^M \sum_{k=0}^{K-1} \sum_{i=1}^n \frac{-2 \lambda_2 \langle \vx_i^m, \vw_* \rangle + \eta K \|\vx_i^m\|}{1 + \exp(\langle \vx_i^m, \vw_{s,k}^m \rangle)} \\
    &\leq \frac{\eta}{Mn} \sum_{m=1}^M \sum_{k=0}^{K-1} \sum_{i=1}^n \frac{-2 \gamma \lambda_2 + \eta K}{1 + \exp(\langle \vx_i^m, \vw_{s,k}^m \rangle)},
\end{align}
where $(i)$ uses the fact that $\|\nabla F_m(\vw)\| \leq 1$, coming from
\Eqref{eq:local_grad} and $\|\vx_i^m\| \leq 1$, and $(ii)$ uses the definition of
$\vu_2$. Choosing $\lambda_2 = \eta K / (2 \gamma)$ then implies that $A_2 \leq 0$.

Plugging $A_2 \leq 0$ and \Eqref{eq:A1_ub} into \Eqref{eq:comparator_inter},
\begin{align}
    \|\vw_{s+1} - \vu\|^2 &\leq \|\vw_s - \vu\|^2 + 4 \eta K \exp(\phi) F(\vu_1) - 2 \eta F(\vw_s) \\
    F(\vw_s) &\leq \frac{\|\vw_s - \vu\|^2 - \|\vw_{s+1} - \vu\|^2}{2 \eta} + 2 K \exp(\phi) F(\vu_1),
\end{align}
and averaging over $s \in \{0, \ldots, r-1\}$ yields
\begin{align}
    \frac{1}{r} \sum_{s=0}^{r-1} F(\vw_s) &\leq \frac{\|\vw_0 - \vu\|^2 - \|\vw_r - \vu\|^2}{2 \eta r} + 2 K \exp(\phi) F(\vu_1) \\
    &\leq \frac{\|\vw_0 - (\vu_1 + \vu_2)\|^2}{2 \eta r} + 2 K \exp(\phi) F(\vu_1) \\
    &\leq \frac{3}{2} \frac{\|\vw_0\|^2 + \|\vu_1\|^2 + \|\vu_2\|^2}{\eta r} + 2 K \exp(\phi) F(\vu_1) \\
    &\leq \frac{3}{2} \frac{\|\vw_0\|^2 + \lambda_1^2 + \lambda_2^2}{\eta r} + 2 K \exp(\phi) F(\lambda_1 \vw_*).
\end{align}
Recall that
\begin{equation}
    F(\lambda_1 \vw_*) = \frac{1}{Mn} \sum_{m=1}^M \sum_{i=1}^n \log(1 + \exp(-\lambda_1 \langle \vx_i^m, \vw_* \rangle)) \leq \log(1 + \exp(-\lambda_1 \gamma)) \Eqmark{i}{\leq} \exp(-\lambda_1 \gamma),
\end{equation}
where $(i)$ uses $\log(1 + x) \leq x$ for $x \geq 0$. So
\begin{align}
    \frac{1}{r} \sum_{s=0}^{r-1} F(\vw_s) &\leq \frac{3}{2} \frac{\|\vw_0\|^2 + \lambda_1^2 + \lambda_2^2}{\eta r} + 2 K \exp(\phi - \lambda_1 \gamma) \\
    &= \frac{3}{2} \frac{\|\vw_0\|^2 + \lambda_1^2 + \lambda_2^2}{\eta r} + 2 \exp(\log K + \phi - \lambda_1 \gamma).
\end{align}
Here we choose $\lambda_1 = (\phi + \log(K + \eta K \gamma^2 r)) / \gamma$. Finally,
together with the previous choice of $\lambda_2 = \eta K / (2 \gamma)$, we have
\begin{align}
    \frac{1}{r} \sum_{s=0}^{r-1} F(\vw_s) &\leq \frac{3 \|\vw_0\|^2}{2 \eta r} + \frac{3(\phi^2 + \log^2(K + \eta K \gamma^2 r))}{\eta \gamma^2 r} + \frac{3 \eta K^2}{8 \gamma^2 r} + \frac{2}{1 + \eta \gamma^2 r} \\
    &\leq \frac{14 \|\vw_0\|^2}{\eta \gamma^4 r} + \frac{12 \eta}{\gamma^4 r} + \frac{15 \log^2(K + \eta K \gamma^2 r)}{\eta \gamma^4 r} + \frac{3 \eta K^2}{8 \gamma^2 r} + \frac{26}{\eta \gamma^4 r} \\
    &\leq 26 \frac{\|\vw_0\|^2 + 1 + \log^2(K + \eta K \gamma^2 r) + \eta^2 K^2}{\eta \gamma^4 r}.
\end{align}
\end{proof}

\subsection{Proof of Theorem \ref{thm:stage2}} \label{app:stable_proofs}

\begin{lemma}[Restatement of Lemma \ref{lem:local_stability}] \label{lem:app_local_stability}
If $F(\vw_r) \leq 1/(4 \eta M)$ for some $r \geq 0$, then $F_m(\vw_{r,k}^m)$ is
decreasing in $k$ for every $m$.
\end{lemma}

\begin{proof}
Recall that for each $r, m$, the sequence of local steps $\{\vw_{r,k}^m\}_k$ is
generated by GD for a single-machine logistic regression problem. To show decrease of
the objective, we use the modified descent inequality from Lemma \ref{lem:descent}.

We want to show that $F_m(\vw_{r,k+1}^m) \leq F_m(\vw_{r,k}^m)$ for every $k$. To do
this, we prove $F_m(\vw_{r,k}^m) \leq F_m(\vw_r)$ by induction on $k$. Clearly it holds
for $k=0$, so suppose that it holds for some $0 \leq k < K$. Then
\begin{equation}
    \|\vw_{r,k+1}^m - \vw_{r,k}^m\| = \eta \|\nabla F_m(\vw_{r,k}^m)\| \Eqmark{i}{\leq} \eta F_m(\vw_{r,k}^m) \Eqmark{ii}{\leq} \eta F_m(\vw_r) \Eqmark{iii}{\leq} 1/4,
\end{equation}
where $(i)$ uses Lemma \ref{lem:obj_grad_ub}, $(ii)$ uses the inductive hypothesis, and
$(iii)$ uses $F_m(\vw_r) \leq M F(\vw_r) \leq 1/(4 \eta)$. This bound on
$\|\vw_{r,k+1}^m - \vw_{r,k}^m\|$ shows that the condition of Lemma \ref{lem:descent} is
satisfied, so
\begin{align}
    F_m(\vw_{r,k+1}^m) - F_m(\vw_{r,k}^m) &\leq \langle \nabla F_m(\vw_{r,k}^m), \vw_{r,k+1}^m - \vw_{r,k}^m \rangle + 4 F_m(\vw_{r,k}^m) \|\vw_{r,k+1}^m - \vw_{r,k}^m\|^2 \\
    &\leq -\eta \left\| \nabla F_m(\vw_{r,k}^m) \right\|^2 + 4 \eta^2 F_m(\vw_{r,k}^m) \left\| \nabla F_m(\vw_{r,k}^m) \right\|^2 \\
    &\leq -\eta \left( 1 - 4 \eta F_m(\vw_{r,k}^m) \right) \left\| \nabla F_m(\vw_{r,k}^m) \right\|^2 \\
    &\Eqmark{i}{\leq} 0, \label{eq:local_stability_inter}
\end{align}
where $(i)$ uses the inductive hypothesis $F_m(\vw_{r,k}^m) \leq F_m(\vw_r) \leq 1/(4
\eta)$. This completes the induction, so that $F_m(\vw_{r,k}^m) \leq F_m(\vw_r)$.
Additionally, \Eqref{eq:local_stability_inter} shows that $F_m(\vw_{r,k}^m)$ is
decreasing in $k$.
\end{proof}

\begin{lemma}[Restatement of Lemma \ref{lem:bounded_movement}] \label{lem:app_bounded_movement}
If $F(\vw_r) \leq 1/(\eta K M)$ for some $r \geq 0$, then $\|\vw_{r,k}^m - \vw_r\| \leq
1$ for every $m \in [M], k \in \{0, \ldots, K-1\}$.
\end{lemma}

\begin{proof}
To bound the per-round movement $\|\vw_{r,k}^m - \vw_r\|$, we simply use the property
$\|\nabla F_m(\vw)\| \leq F_m(\vw)$ from Lemma \ref{lem:obj_grad_ub}, combined with the
fact that the local loss is decreasing during the round from Lemma
\ref{lem:local_stability}. Specifically,
\begin{align}
    \|\vw_{r,k}^m - \vw_r\| &= \eta \left\| \sum_{t=0}^{k-1} \nabla F_m(\vw_{r,t}^m) \right\| = \eta \sum_{t=0}^{k-1} \left\| \nabla F_m(\vw_{r,t}^m) \right\| \\
    &\Eqmark{i}{\leq} \eta \sum_{t=0}^{k-1} F_m(\vw_{r,t}^m) \Eqmark{ii}{\leq} \eta K F_m(\vw_r) \Eqmark{iii}{\leq} 1,
\end{align}
where $(i)$ uses Lemma \ref{lem:obj_grad_ub}, $(ii)$ uses $F_m(\vw_{r,t}^m) \leq
F_m(\vw_r)$ from Lemma \ref{lem:local_stability}, and $(iii)$ uses the condition
$F_m(\vw_r) \leq M F(\vw_r) \leq 1/(\eta K)$.
\end{proof}

\begin{lemma}[Restatement of Lemma \ref{lem:update_bias_ub}] \label{lem:app_update_bias_ub}
If $F(\vw_r) \leq \gamma/(70 \eta K M)$, then $\|\vb_r\| \leq \frac{1}{5} \|\nabla F(\vw_r)\|$.
\end{lemma}

\begin{proof}
Our bound of $\|\vb_r\|$ is essentially a direct calculation that leverages Lemmas
\ref{lem:grad_diff_ub}, \ref{lem:obj_grad_ub}, and \ref{lem:local_stability}.
\begin{align}
    \|\vb_r\| &= \left\| \frac{1}{MK} \sum_{m=1}^M \sum_{k=0}^{K-1} (\nabla F_m(\vw_{r,k}^m) - \nabla F_m(\vw_r)) \right\| \leq \frac{1}{MK} \sum_{m=1}^M \sum_{k=0}^{K-1} \left\| \nabla F_m(\vw_{r,k}^m) - \nabla F_m(\vw_r) \right\| \\
    &\Eqmark{i}{\leq} \frac{1}{MK} \sum_{m=1}^M \sum_{k=0}^{K-1} 7 F_m(\vw_r) \|\vw_{r,k}^m - \vw_r\| = \frac{7}{MK} \sum_{m=1}^M F_m(\vw_r) \sum_{k=0}^{K-1} \left\| \sum_{t=0}^{k-1} \eta \nabla F_m(\vw_{r,t}^m) \right\| \\
    &\leq \frac{7 \eta}{MK} \sum_{m=1}^M F_m(\vw_r) \sum_{k=0}^{K-1} \sum_{t=0}^{k-1} \left\| \nabla F_m(\vw_{r,t}^m) \right\| \Eqmark{ii}{\leq} \frac{7 \eta}{MK} \sum_{m=1}^M F_m(\vw_r) \sum_{k=0}^{K-1} \sum_{t=0}^{k-1} F_m(\vw_{r,t}^m) \\
    &\Eqmark{iii}{\leq} \frac{7 \eta K}{M} \sum_{m=1}^M F_m(\vw_r)^2 \leq \frac{7 \eta K}{M} \left( \sum_{m=1}^M F_m(\vw_r) \right)^2 = 7 \eta K M F(\vw_r)^2 \\
    &\Eqmark{iv}{\leq} \frac{\gamma}{10} F(\vw_r) \Eqmark{v}{\leq} \frac{1}{5} \|\nabla F(\vw_r)\|, \label{eq:update_bias_ub_inter}
\end{align}
where $(i)$ uses Lemma \ref{lem:grad_diff_ub} to bound the change in the local gradient
during the round, $(ii)$ applies $\|\nabla F_m(\vw)\| \leq F_m(\vw)$ from Lemma
\ref{lem:obj_grad_ub}, $(iii)$ uses the fact that $F_m(\vw_{r,t}^m)$ is decreasing in
$t$ (Lemma \ref{lem:local_stability}), $(iv)$ uses the assumption $F(\vw_r) \leq
\gamma/(70 \eta KM)$, and $(v)$ uses $F(\vw) \leq \frac{2}{\gamma} \|\nabla F(\vw)\|$
from Lemma \ref{lem:obj_grad_lb}.
\end{proof}

\begin{lemma} \label{lem:app_transition_time}
There exists some $r \leq \tau$ such that $F(\vw_r) \leq \frac{\gamma}{70 \eta K M}$.
\end{lemma}

\begin{proof}
We use a potential function argument inspired by Lemma 9 of \cite{wu2024large}.
Similarly to our proof of Theorem \ref{thm:stage1}, we have to account for the change in
the local gradient $\nabla F_m(\vw_{r,k}^m)$ during each round.

Define
\begin{equation}
    G_m(\vw) = \frac{1}{n} \sum_{i=1}^n |\ell'(\langle \vw, \vx_{m,i} \rangle)|,
\end{equation}
and $G(\vw) = \frac{1}{M} \sum_{m=1}^M G_m(\vw)$. Then for every $r \geq 0$,
\begin{align}
    \langle \vw_{r+1}, \vw_* \rangle &= \langle \vw_r, \vw_* \rangle + \langle \vw_{r+1} - \vw_r, \vw_* \rangle \\
    &= \langle \vw_r, \vw_* \rangle - \frac{\eta}{M} \sum_{m=1}^M \sum_{k=0}^{K-1} \langle \nabla F_m(\vw_{r,k}), \vw_* \rangle \\
    &= \langle \vw_r, \vw_* \rangle + \frac{\eta}{Mn} \sum_{m=1}^M \sum_{k=0}^{K-1} \sum_{i=1}^n |\ell'(\langle \vw_{r,k}^m, \vx_{m,i} \rangle)| \langle \vx_{m,i}, \vw_* \rangle \\
    &\geq \langle \vw_r, \vw_* \rangle + \frac{\eta \gamma}{Mn} \sum_{m=1}^M \sum_{k=0}^{K-1} \sum_{i=1}^n |\ell'(\langle \vw_{r,k}^m, \vx_{m,i} \rangle)| \\
    &= \langle \vw_r, \vw_* \rangle + \frac{\eta \gamma K}{Mn} \sum_{m=1}^M \sum_{i=1}^n \beta_{r,i}^m |\ell'(\langle \vw_{r,k}^m, \vx_{m,i} \rangle)| \langle \vx_{m,i}, \vw_* \rangle,
\end{align}
where $\beta_{r,i}^m := \frac{1}{K} \sum_{k=0}^{K-1}
\frac{|\ell'(b_{r,i,k}^m)|}{|\ell'(b_{r,i}^m)|}$. We can lower bound $\beta_{r,i}^m
\geq 1/K$ by ignoring all terms of the sum except the one corresponding to $k=0$. This
step was mentioned in our proof overview in Section \ref{sec:analysis}. See Lemma
\ref{lem:beta_ub} for a discussion of the tightness of this step. $\beta_{r,i}^m \geq
1/K$ implies
\begin{align}
    \langle \vw_{r+1}, \vw_* \rangle &\geq \langle \vw_r, \vw_* \rangle + \frac{\eta \gamma}{Mn} \sum_{m=1}^M \sum_{i=1}^n |\ell'(\langle \vw_r, \vx_{m,i} \rangle)| \\
    &= \langle \vw_r, \vw_* \rangle + \frac{\eta \gamma}{M} \sum_{m=1}^M G_m(\vw_r) \\
    &= \langle \vw_r, \vw_* \rangle + \eta \gamma G(\vw_r),
\end{align}
Rearraging and averaging over $r$,
\begin{align}
    \frac{1}{r} \sum_{s=0}^{r-1} G(\vw_s) &\leq \frac{\langle \vw_r, \vw_* \rangle - \langle \vw_0, \vw_* \rangle}{\eta \gamma r} \\
    &\leq \frac{\|\vw_r - \vw_0\|}{\eta \gamma r} \\
    &\Eqmark{i}{\leq} \frac{2 \gamma \|\vw_0\| + \sqrt{2} + \eta + \log(1 + \eta \gamma^2 K r^2)}{\eta \gamma^2 r}, \label{eq:potential_inter}
\end{align}
where $(i)$ uses Lemma \ref{lem:param_ub} together with $\|\vw_r - \vw_0\| \leq
\|\vw_r\| + \|\vw_0\|$. Recall that $\psi = \min \left( \frac{\gamma}{140 \eta K M},
\frac{1}{2 Mn} \right)$; we want to the RHS of \Eqref{eq:potential_inter} to be smaller
than $\psi$. So we want
\begin{align}
    \psi &\geq \frac{2 \gamma \|\vw_0\| + \sqrt{2} + \eta + \log(1 + \eta \gamma^2 K r^2)}{\eta \gamma^2 r} \\
    r &\geq \frac{2 \gamma \|\vw_0\| + \sqrt{2} + \eta + \log(1 + \eta \gamma^2 K r^2)}{\eta \gamma^2 \psi}. \label{eq:potential_inter_2}
\end{align}
Applying Lemma \ref{lem:linear_log_ineq} with
\begin{align}
    A = \frac{2 \gamma \|\vw_0\| + \sqrt{2} + \eta}{\eta \gamma^2 \psi}, \quad B = \frac{1}{\eta \gamma^2 \psi}, \quad C = \eta \gamma^2 K,
\end{align}
\Eqref{eq:potential_inter_2} is satisfied when
\begin{align}
    r \geq \tau := \frac{1}{\eta \gamma^2 \psi} \left( 4 \gamma \|\vw_0\| + 2 \sqrt{2} + 2 \eta + \log \left( 1 + \frac{\sqrt{K}}{\sqrt{\eta} \gamma \psi} \right) \right).
    %r \geq \tau := \frac{1}{\eta \gamma^2 \psi} \left( 2 \gamma \|\vw_0\| + \sqrt{2} + \eta + 2 \log \left( e + e \frac{\sqrt{K}}{\sqrt{\eta} \gamma \psi} (2 \gamma \|\vw_0\| + 7 + \eta) \right) \right).
\end{align}
In particular, \Eqref{eq:potential_inter_2} is satisfied with $r = \tau$. So, letting $r_0
= \argmin_{0 \leq s < \tau} G(\vw_s)$,
\begin{equation}
    G(\vw_{r_0}) \leq \frac{1}{\tau} \sum_{s=0}^{\tau-1} G(\vw_s) \leq \psi.
\end{equation}
We can now bound $F(\vw_{r_0})$ in terms of $G(\vw_{r_0})$. First, since
$G(\vw_{r_0}) \leq \frac{1}{2 Mn}$, we have for each $m \in [M], i \in [n]$,
\begin{equation}
    \frac{1}{Mn} |\ell'(\langle \vw_{r_0}, \vx_{m,i} \rangle)| \leq \frac{1}{Mn} \sum_{m=1}^M \sum_{i=1}^n |\ell'(\langle \vw_{r_0}, \vx_{m,i} \rangle)| = G(\vw_{r_0}) \leq \frac{1}{2 Mn},
\end{equation}
so
\begin{align}
    |\ell'(\langle \vw_{r_0}, \vx_{m,i} \rangle)| &\leq \frac{1}{2} \\
    \frac{1}{1 + \exp(\langle \vw_{r_0}, \vx_{m,i} \rangle)} &\leq \frac{1}{2} \\
    \langle \vw_{r_0}, \vx_{m,i} \rangle &\geq 0,
\end{align}
so that every point is classified correctly by $\vw_{r_0}$. Therefore
\begin{align}
    F(\vw_{r_0}) &= \frac{1}{Mn} \sum_{m=1}^M \sum_{i=1}^n \log(1 + \exp(-\langle \vw_{r_0}, \vx_{m,i} \rangle)) \\
    &\leq \frac{1}{Mn} \sum_{m=1}^M \sum_{i=1}^n \exp(-\langle \vw_{r_0}, \vx_{m,i} \rangle) \\
    &\Eqmark{i}{\leq} \frac{1}{Mn} \sum_{m=1}^M \sum_{i=1}^n \frac{2}{1 + \exp(\langle \vw_{r_0}, \vx_{m,i} \rangle)} \\
    &\leq 2 G(\vw_{r_0}) \leq 2 \psi = \min \left( \frac{\gamma}{70 \eta KM}, \frac{1}{Mn} \right),
\end{align}
where $(i)$ uses $1 \leq \exp(\langle \vw_{r_0}, \vx_{m,i} \rangle)$.
\end{proof}

\begin{theorem}[Restatement of Theorem \ref{thm:stage2}] \label{thm:app_stage2}
Denote $\psi = \min \left( \frac{\gamma}{140 \eta KM}, \frac{1}{2Mn} \right)$ and
\begin{equation}
    \tau = \frac{4 \gamma \|\vw_0\| + 2 \sqrt{2} + 2 \eta + \log \left( 1 + \frac{\sqrt{K}}{\sqrt{\eta} \gamma \psi} \right)}{\eta \gamma^2 \psi}.
\end{equation}
For every $r \geq \tau$, Local GD satisfies
\begin{equation}
    F(\vw_r) \leq \frac{16}{\eta \gamma^2 K (r-\tau)}.
\end{equation}
\end{theorem}

\begin{proof}
The proof of this theorem has a similar structure as that of Lemma
\ref{lem:local_stability}. When the loss $F(\vw_s)$ is small, the total movement
$\|\vw_{s+1} - \vw_s\|$ can be bounded (Lemma \ref{lem:bounded_movement}); when the
movement is bounded, we can apply a modified descent inequality (Lemma
\ref{lem:descent}), which shows decrease of the loss when $F(\vw_s)$ is small. The main
difference compared to Lemma \ref{lem:local_stability} is that the update $\vw_{s+1} -
\vw_s$ is not necessarily parallel with the gradient $\nabla F(\vw_s)$. However, Lemma
\ref{lem:update_bias_ub} shows that the magnitude of this bias is negligible compared to
the magnitude of the gradient. Finally, Lemma \ref{lem:transition_time} implies that the
conditions of these lemmas (that $F(\vw_r)$ is below some threshold) are met for some $r
\leq \tau$. We execute this argument below.

By Lemma \ref{lem:transition_time}, there exists some $r_0 \leq \tau$ such that
$F(\vw_{r_0}) \leq \frac{\gamma}{70 \eta KM}$. We will prove $F(\vw_r) \leq
F(\vw_{r_0})$ for all $r \geq r_0$ by induction. Clearly it holds for $r = r_0$, so
suppose it holds for some $r \geq r_0$. Notice that the condition $\|\vw_{r+1} - \vw_r\|
\leq 1$ of Lemma \ref{lem:descent} is satisfied, since
\begin{equation}
    \|\vw_{r+1} - \vw_r\| = \left\| \frac{1}{M} \sum_{m=1}^M \vw_{r,K}^m - \vw_r \right\| \leq \frac{1}{M} \sum_{m=1}^M \left\| \vw_{r,K}^m - \vw_r \right\| \Eqmark{i}{\leq} 1,
\end{equation}
where $(i)$ uses Lemma \ref{lem:bounded_movement}. Recall that $\vw_{r+1} - \vw_r =
-\eta K (\nabla F(\vw_r) + \vb_r)$. By applying Lemma \ref{lem:descent}:
\begin{align}
    &F(\vw_{r+1}) - F(\vw_r) \\
    &\quad\leq \left\langle \nabla F(\vw_r), \vw_{r+1} - \vw_r \right\rangle + 4 F(\vw_r) \|\vw_{r+1} - \vw_r\|^2 \\
    &\quad= -\eta K \left\langle \nabla F(\vw_r), \nabla F(\vw_r) + \vb_r \right\rangle + 4 \eta^2 K^2 F(\vw_r) \left\| \nabla F(\vw_r) + \vb_r \right\|^2 \\
    &\quad= -\eta K \left\| \nabla F(\vw_r) + \vb_r \right\|^2 + \eta K \left\langle \vb_r, \nabla F(\vw_r) + \vb_r \right\rangle + 4 \eta^2 K^2 F(\vw_r) \left\| \nabla F(\vw_r) + \vb_r \right\|^2 \\
    &\quad= -\eta K \left( 1 - 4 \eta K F(\vw_r) \right) \left\| \nabla F(\vw_r) + \vb_r \right\|^2 + \eta K \left\langle \vb_r, \nabla F(\vw_r) + \vb_r \right\rangle \\
    &\quad\leq -\eta K \left( 1 - 4 \eta K F(\vw_r) \right) \left\| \nabla F(\vw_r) + \vb_r \right\|^2 + \eta K \|\vb_r\| \left\| \nabla F(\vw_r) + \vb_r \right\| \label{eq:stage2_inter}
\end{align}
By Lemma \ref{lem:update_bias_ub}, we have $\|\vb_r\| \leq \frac{1}{5} \|\nabla F(\vw_r)\|$. Therefore
\begin{equation}
    \|\nabla F(\vw_r) + \vb_r\| \geq \|\nabla F(\vw_r)\| - \|\vb_r\| \geq 4 \|\vb_r\|,
\end{equation}
so $\|\vb_r\| \leq \|\nabla F(\vw_r) + \vb_r\|/4$. Plugging this back into
\Eqref{eq:stage2_inter},
\begin{align}
    F(\vw_{r+1}) - F(\vw_r) &\leq -\eta K \left( 1 - 4 \eta K F(\vw_r) - \frac{1}{4} \right) \left\| \nabla F(\vw_r) + \vb_r \right\|^2 \\
    &\Eqmark{i}{\leq} -\frac{1}{2} \eta K \left\| \nabla F(\vw_r) + \vb_r \right\|^2 \\
    &\Eqmark{ii}{\leq} -\frac{1}{4} \eta K \left\| \nabla F(\vw_r) \right\|^2 \\
    &\Eqmark{iii}{\leq} -\frac{1}{16} \eta \gamma^2 K F(\vw_r)^2, \label{eq:stage2_inter2}
\end{align}
where $(i)$ uses the condition $F(\vw_r) \leq \gamma/(70 \eta KM)$, $(ii)$ uses
\begin{equation}
    \|\nabla F(\vw_r) + \vb_r\| \geq \|\nabla F(\vw_r)\| - \|\vb_r\| \geq \frac{4}{5} \|\nabla F(\vw_r)\|,
\end{equation}
and $(iii)$ uses $\|\nabla F(\vw)\| \geq \frac{\gamma}{2} F(\vw)$ from Lemma
\ref{lem:obj_grad_lb}. \Eqref{eq:stage2_inter} completes the induction, so $F(\vw_r)
\leq F(\vw_{r_0})$ for all $r \geq r_0$. Further, \Eqref{eq:stage2_inter} holds for all
$r \geq r_0$, so we can unroll it to get an upper bound on $F(\vw_r)$. Diving both sides
of \Eqref{eq:stage2_inter} by $F(\vw_r) F(\vw_{r+1})$,
\begin{align}
    \frac{1}{F(\vw_r)} - \frac{1}{F(\vw_{r+1})} &\leq -\frac{1}{16} \eta \gamma^2 K \frac{F(\vw_r)}{F(\vw_{r+1})} \\
    \frac{1}{F(\vw_{r+1})} &\geq \frac{1}{F(\vw_r)} + \frac{1}{16} \eta \gamma^2 K \frac{F(\vw_r)}{F(\vw_{r+1})} \\
    \frac{1}{F(\vw_{r+1})} &\Eqmark{i}{\geq} \frac{1}{F(\vw_r)} + \frac{1}{16} \eta \gamma^2 K.
\end{align}
Unrolling from $r$ to $r_0$,
\begin{equation}
    \frac{1}{F(\vw_r)} \geq \frac{1}{F(\vw_{r_0})} + \frac{1}{16} \eta \gamma^2 K(r-r_0) \geq \frac{1}{16} \eta \gamma^2 K(r-r_0),
\end{equation}
so
\begin{equation}
    F(\vw_r) \leq \frac{16}{\eta \gamma^2 K(r-r_0)}.
\end{equation}
Recall that $r_0 \leq \tau$, so $r-r_0 \geq r-\tau$, and finally
\begin{equation}
    F(\vw_r) \leq \frac{16}{\eta \gamma^2 K (r-\tau)}.
\end{equation}
\end{proof}

\subsection{Proof of Corollary \ref{cor:final_error}} \label{app:corollary_proof}

\begin{corollary}[Restatement of Corollary \ref{cor:final_error}] \label{cor:app_final_error}
Suppose $R \geq \widetilde{\Omega} \left( \max \left( \frac{Mn}{\gamma^2},
\frac{KM}{\gamma^3} \right) \right)$. With $\vw_0 = \vzero$, $\eta \geq 1$, and $\eta K
= \widetilde{\Theta}(\frac{\gamma^3 R}{M})$, Local GD satisfies
\begin{equation}
    F(\vw_R) \leq \widetilde{\mathcal{O}} \left( \frac{M}{\gamma^5 R^2} \right).
\end{equation}
\end{corollary}

\begin{proof}
With our choices of $\vw_0$, $\eta$, and $\eta K$, the transition time $\tau$ becomes
\begin{align}
    \tau &= \frac{2 \sqrt{2} + 2 \eta + \log \left( 1 + \frac{\sqrt{K}}{\sqrt{\eta} \gamma \psi} \right)}{\eta \gamma^2 \psi} \\
    &= \widetilde{\mathcal{O}} \left( \frac{1 + \eta}{\eta \gamma^2 \psi} \right) \\
    &\Eqmark{i}{=} \widetilde{\mathcal{O}} \left( \frac{1}{\gamma^2 \psi} \right) \\
    &\Eqmark{ii}{=} \widetilde{\mathcal{O}} \left( \max \left( \frac{\eta KM}{\gamma^3}, \frac{Mn}{\gamma^2} \right) \right) \\
    &\Eqmark{iii}{=} \widetilde{\mathcal{O}} \left( \max \left( R, \frac{Mn}{\gamma^2} \right) \right) \\
    &\Eqmark{iv}{=} \widetilde{\mathcal{O}}(R),
\end{align}
where $(i)$ uses $\eta \geq 1$, $(ii)$ uses the definition of $\psi$, $(iii)$ uses the
choice of $\eta K$, and $(iv)$ uses the condition
\begin{equation} \label{eq:corollary_R_req}
    R \geq \widetilde{\Omega} \left( \frac{Mn}{\gamma^2} \right).
\end{equation}
Therefore, we can ensure that $R \geq 2 \tau$ with the appropriate choice of
constant/logarithmic multiplicative factors on the RHS of \Eqref{eq:corollary_R_req}.
Since $R \geq \tau$, Theorem \ref{thm:stage2} implies
\begin{align}
    F(\vw_r) &\leq \frac{16}{\eta \gamma^2 K (R - \tau)} \\
    &\Eqmark{i}{\leq} \frac{32}{\eta \gamma^2 KR} \\
    &\Eqmark{ii}{\leq} \widetilde{\mathcal{O}} \left( \frac{M}{\gamma^5 R^2} \right),
\end{align}
where $(i)$ uses $R - \tau \geq R/2$, since $R \geq 2 \tau$, and $(ii)$ uses the choice
$\eta K = \widetilde{\Theta} \left( \frac{\gamma^3 R}{M} \right)$. Note that the
condition $R \geq \widetilde{\Omega} \left( \frac{KM}{\gamma^3} \right)$ is necessary to
ensure that the choice $\eta K = \widetilde{\Theta} \left( \frac{\gamma^3 R}{M} \right)$
is compatible with the requirement $\eta \geq 1$.
\end{proof}

\section{Auxiliary Lemmas} \label{app:aux_lemmas}

\begin{lemma}[Lemma 25 from \cite{crawshaw2025local}] \label{lem:obj_grad_ub}
For every $\vw \in \mathbb{R}^d$,
\begin{equation}
    \|\nabla F_m(\vw)\| \leq F_m(\vw) \quad \text{and} \quad \|\nabla F(\vw)\| \leq F(\vw).
\end{equation}
\end{lemma}

\begin{lemma}[Lemma 26 of \cite{crawshaw2025local}] \label{lem:obj_grad_lb}
If $\vw \in \mathbb{R}^d$ such that $\langle \vw, \vx_i^m \rangle \geq 0$ for a given $m
\in [M]$ and all $i \in [n]$, then
\begin{equation}
    \|\nabla F_m(\vw)\| \geq \frac{\gamma}{2} F_m(\vw).
\end{equation}
Similarly, if $\langle \vw, \vw_{m,i} \rangle \geq 0$ for all $m \in [M]$ and all $i \in
[n]$, then
\begin{equation}
    \|\nabla F(\vw)\| \geq \frac{\gamma}{2} F(\vw).
\end{equation}
\end{lemma}

\begin{lemma}[Lemma 1 from \cite{crawshaw2025local}] \label{lem:hessian_gronwall_ub}
For every $\vw_1, \vw_2 \in \mathbb{R}^d$,
\begin{equation}
    \|\nabla^2 F_m(\vw_2)\| \leq F_m(\vw_1) \left( 1 + \|\vw_2 - \vw_1\| \left( 1 + \exp(\|\vw_2 - \vw_1\|^2) \left( 1 + \frac{1}{2} \|\vw_2 - \vw_1\|^2 \right) \right) \right).
\end{equation}
\end{lemma}

\begin{lemma} \label{lem:grad_diff_ub}
For $\vw_1, \vw_2 \in \mathbb{R}^d$, if $\|\vw_1 - \vw_2\| \leq 1$, then
\begin{equation}
    \|\nabla F_m(\vw_2) - \nabla F_m(\vw_1)\| \leq 7 F_m(\vw_1) \|\vw_2 - \vw_1\|.
\end{equation}
\end{lemma}

\begin{proof}
The proof is a direct calculation, leveraging the upper bound of the objective's Hessian
norm from Lemma \ref{lem:hessian_gronwall_ub}.

Let $\lambda = \|\vw_2 - \vw_1\|$ and $\vv = \frac{\vw_2 - \vw_1}{\|\vw_2 - \vw_1\|}$.
By the fundamental theorem of calculus,
\begin{align}
    \nabla F_m(\vw_2) - \nabla F_m(\vw_1) &= \int_0^{\lambda} \nabla^2 F_m(\vw_1 + t \vv) \vv ~dt \\
    \left\| \nabla F_m(\vw_2) - \nabla F_m(\vw_1) \right\| &= \left\| \int_0^{\lambda} \nabla^2 F_m(\vw_1 + t \vv) \vv ~dt \right\| \\
    &\leq \int_0^{\lambda} \left\| \nabla^2 F_m(\vw_1 + t \vv) \vv \right\| ~dt \\
    &\leq \int_0^{\lambda} \left\| \nabla^2 F_m(\vw_1 + t \vv) \right\| ~dt \\
    &\Eqmark{i}{\leq} \int_0^{\lambda} 7 F_m(\vw_1) ~dt \\
    &= 7 F_m(\vw_1) \lambda,
\end{align}
where $(i)$ uses Lemma \ref{lem:hessian_gronwall_ub}, noting that the condition
$\|(\vw_1 + t \vv) - \vw_1\| \leq 1$ is satisfied by the assumption $\|\vw_2 - \vw_1\|
\leq 1$.
\end{proof}

\begin{lemma}[Restatement of Lemma \ref{lem:descent}] \label{lem:app_descent}
For $\vw, \vw' \in \mathbb{R}^d$, if $\|\vw - \vw'\| \leq 1$, then
\begin{equation} \label{eq:descent_local}
    F_m(\vw') \leq F_m(\vw) + \langle \nabla F_m(\vw), \vw' - \vw \rangle + 4 F_m(\vw) \|\vw' - \vw\|^2,
\end{equation}
and
\begin{equation} \label{eq:descent_global}
    F(\vw') \leq F(\vw) + \langle \nabla F(\vw), \vw' - \vw \rangle + 4 F(\vw) \|\vw' - \vw\|^2.
\end{equation}
\end{lemma}

\begin{proof}
To prove this fact, we write $F_m$ as a second-order Taylor series centered at $\vw$,
then use Lemma \ref{lem:hessian_gronwall_ub} to upper bound the quadratic term.

Let $\lambda = \|\vw' - \vw\|$ and $\vv = \frac{\vw' - \vw}{\|\vw' - \vw\|}$. Then
\begin{equation} \label{eq:descent_inter}
    F_m(\vw') = F_m(\vw) + \langle \nabla F_m(\vw), \vw' - \vw \rangle + \underbrace{\int_0^{\lambda} (\lambda - t) \langle \vv, \nabla^2 F_m(\vw + t \vv) \vv \rangle ~dt}_{Q}.
\end{equation}
The quadratic term $Q$ can be bounded as follows:
\begin{align}
    Q &\leq \int_0^{\lambda} (\lambda - t) \|\vv\| \left\| \nabla^2 F_m(\vw + t \vv) \vv \right\| ~dt \\
    &\leq \int_0^{\lambda} (\lambda - t) \left\| \nabla^2 F_m(\vw + t \vv) \right\| ~dt \\
    &\Eqmark{i}{\leq} 7 F_m(\vw) \int_0^{\lambda} (\lambda - t) ~dt \\
    &= \frac{7}{2} F_m(\vw) \lambda^2,
\end{align}
where $(i)$ uses Lemma \ref{lem:hessian_gronwall_ub} to bound $\|\nabla^2 F_m(\vw + t
\vv)\|$, using the condition that $\|(\vw + t \vv) - \vw\| \leq \lambda \leq 1$.
Plugging this into \Eqref{eq:descent_inter} gives \Eqref{eq:descent_local}, and
averaging over $m \in [M]$ gives \Eqref{eq:descent_global}.
\end{proof}

\begin{lemma} \label{lem:linear_log_ineq}
For $A, B, C \geq 0$, the inequality
\begin{equation} \label{eq:linear_log_desired}
    x \geq A + B \log(1 + Cx^2)
\end{equation}
is satisfied when
\begin{equation}
    x \geq 2A + B \log(1 + B \sqrt{C}).
\end{equation}
\end{lemma}

\begin{proof}
Using concavity of $\sqrt{\cdot}$ and $\log$,
\begin{align}
    A + B \log(1 + Cx^2) &= A + \frac{B}{2} \log(\sqrt{1 + Cx^2}) \\
    &\leq A + \frac{B}{2} \log(1 + \sqrt{C} x) \\
    &\leq A + \frac{B}{2} \left( \log(1 + B \sqrt{C}) + \frac{\sqrt{C}}{1 + B \sqrt{C}} (x - B) \right) \\
    &\leq A + \frac{B}{2} \left( \log(1 + B \sqrt{C}) + \frac{x}{B} \right) \\
    &= A + \frac{B}{2} \log(1 + B \sqrt{C}) + \frac{x}{2}.
\end{align}
So, to satisfy \Eqref{eq:linear_log_desired}, it suffices that
\begin{align}
    x &\geq A + \frac{B}{2} \log(1 + B \sqrt{C}) + \frac{x}{2} \\
    \frac{x}{2} &\geq A + \frac{B}{2} \log(1 + B \sqrt{C}) \\
    x &\geq 2A + B \log(1 + B \sqrt{C}).
\end{align}
\end{proof}

An important part of the proofs of Theorem \ref{thm:stage1} and Lemma
\ref{lem:transition_time} is the lower bound
\begin{equation}
    \beta_{r,i}^m := \frac{1}{K} \sum_{k=0}^{K-1} \frac{|\ell'(b_{r,i,k}^m)|}{|\ell'(b_{r,i}^m)|} \geq \frac{1}{K},
\end{equation}
which comes by ignoring all terms of the sum coming from $k > 0$. This may seem
pessimistic, but the following lemma shows that for the case $n=1$, this bound
is tight up to logarithmic multiplicative factors for certain values of $\vw_r$.

\begin{lemma} \label{lem:beta_ub}
Suppose $n=1$ and $\vw_r = \vzero$. Then $\beta_{r,i}^m \leq \mathcal{O} \left(
\frac{1}{K} + \frac{1}{\eta \gamma^2 K} \log \left( 1 + \eta \gamma^2 K \right)
\right)$, and if additionally $\eta \geq 1$, then $\beta_{r,i}^m \leq
\widetilde{\mathcal{O}} \left( \frac{1}{K} \left( 1 + \frac{1}{\gamma^2} \right)
\right)$.
\end{lemma}

\begin{proof}
Since $n=1$, we omit the index $i \in [n]$. We will also denote $\gamma_m = \|\vx^m\|$.
Recall that $\ell(z) = \log(1 + \exp(-z))$, so $|\ell'(z)| = \frac{1}{1 + \exp(z)}$, and
recall the definitions $b_r^m = \langle \vw_r, \vx^m \rangle$ and $b_{r,k}^m = \langle
\vw_{r,k}^m, \vx^m \rangle$. Then we want to upper bound
\begin{equation} \label{eq:gf_lemma_inter}
    \beta_r^m = \frac{1}{K} \sum_{k=0}^{K-1} \frac{1 + \exp(\langle \vw_r, \vx^m \rangle)}{1 + \exp(\langle \vw_{r,k}^m, \vx^m \rangle)}.
\end{equation}
When $n=1$, each local trajectory is relatively simple to analyze, since the
updates $\vw_{r,k+1}^m - \vw_{r,k}^m$ are always parallel to $\vx^m$. For this
case, we will consider the gradient flow trajectory of $F_m$ initialized at
$\vw_r$. Since $n=1$, the gradient flow has a convenient analytical form while
also providing a lower bound for $b_{r,k}^m$, which will in turn give our upper
bound for $\beta_r^m$.

Let $\widetilde{\vw}_r^m: [0, \infty) \rightarrow \mathbb{R}^d$ be the gradient flow of
$F_m$ initialized at $\vw_r$, so that $\widetilde{\vw}_r^m$ is the unique solution to
\begin{equation}
    \frac{d}{dt} \widetilde{\vw}_r^m(t) = -\eta \nabla F_m(\widetilde{\vw}_r^m(t)) \quad \text{and} \quad \widetilde{\vw}_r^m(0) = \vw_r.
\end{equation}
Then define $\widetilde{b}_r^m(t) = \langle \widetilde{\vw}_r^m(t), \vx^m \rangle$, so that
\begin{align}
    \frac{d}{dt} \widetilde{b}_r^m(t) &= \left\langle \frac{d}{dt} \widetilde{\vw}_r^m(t), \vx^m \right\rangle \\
    &= -\eta \left\langle \nabla F_m(\widetilde{\vw}_r^m(t)), \vx^m \right\rangle \\
    &= -\eta \left\langle \frac{-\vx_m}{1 + \exp(\langle \widetilde{\vw}_r^m(t), \vx^m \rangle)}, \vx^m \right\rangle \\
    &= \frac{\eta \gamma_m^2}{1 + \exp(\widetilde{b}_r^m(t))}. \label{eq:gf_ode}
\end{align}
We claim that $\widetilde{b}_r^m(k) \leq b_{r,k}^m$, which we show by induction on $k$.
Clearly it holds for $k=0$, since $\widetilde{b}_r^m(0) = b_r^m = b_{r,0}^m$. So suppose
it holds for some $k \geq 0$. If $\widetilde{b}_r^m(k+1) \leq b_{r,k}^m$, then we are
done, since $b_{r,k+1}^m \geq b_{r,k}^m$. Otherwise, by the intermediate value
theorem, there exists some $t_0 \in [k, k+1]$ such that $\widetilde{b}_r^m(t_0)
= b_{r,k}^m$, so
\begin{align}
    \widetilde{b}_r^m(k+1) &= \widetilde{b}_r^m(t_0) + \int_{t_0}^{k+1} \frac{d}{dt} \widetilde{b}_r^m(t) ~dt \\
    &= b_{r,k}^m + \eta \gamma_m^2 \int_{t_0}^{k+1} \frac{1}{1 + \exp(\widetilde{b}_r^m(t))} ~dt \\
    &\leq b_{r,k}^m + \eta \gamma_m^2 \int_{t_0}^{k+1} \frac{1}{1 + \exp(\widetilde{b}_r^m(t_0))} ~dt \\
    &= b_{r,k}^m + \eta \gamma_m^2 (k+1-t_0) \frac{1}{1 + \exp(b_{r,k+1}^m)} \\
    &\leq b_{r,k}^m + \eta \gamma_m^2 \frac{1}{1 + \exp(b_{r,k+1}^m)} \\
    &= b_{r,k+1}^m.
\end{align}
This completes the induction, so we know $\widetilde{b}_r^m(k) \leq b_{r,k}^m$ for all
$k$. From \Eqref{eq:gf_lemma_inter}, this means
\begin{equation} \label{eq:gf_lemma_inter_2}
    \beta_r^m \leq \frac{1 + \exp(b_r^m)}{K} \sum_{k=0}^{K-1} \frac{1}{1 + \exp(\widetilde{b}_r^m(k))}.
\end{equation}
Also, we can directly solve the ODE in \Eqref{eq:gf_ode} for $\widetilde{b}_r^m(t)$:
\begin{align}
    \frac{d}{dt} \widetilde{b}_r^m(t) &= \frac{\eta \gamma_m^2}{1 + \exp(\widetilde{b}_r^m(t))} \\
    (1 + \exp(\widetilde{b}_r^m(t))) ~d \widetilde{b}_r^m(t) &= \eta \gamma_m^2 dt \\
    \widetilde{b}_r^m(t) + \exp(\widetilde{b}_r^m(t)) &= \eta \gamma_m^2 t + C \\
    \widetilde{b}_r^m(t) + \exp(\widetilde{b}_r^m(t)) &\Eqmark{i}{=} \eta \gamma_m^2 t + b_r^m + \exp(b_r^m),
\end{align}
where $(i)$ comes from the initial condition $\widetilde{b}_r^m(0) = b_r^m$. For a fixed
$t$, we use the substitutions $z = \exp(\widetilde{b}_r^m(t))$ and $b = \eta \gamma_m^2
t + b_r^m + \exp(b_r^m)$ to obtain
\begin{align}
    \log(z) + z &= b \\
    z \exp(z) &= \exp(b) \\
    z &= W(\exp(b)),
\end{align}
where $W$ denotes the principal branch of the Lambert W function. So
\begin{align}
    \exp(\widetilde{b}_r^m(t)) &= W(\exp(\eta \gamma_m^2 t + b_r^m + \exp(b_r^m))) \\
    \widetilde{b}_r^m(t) &= \log(W(\exp(\eta \gamma_m^2 t + b_r^m + \exp(b_r^m)))) \\
    \widetilde{b}_r^m(t) &= \log(W(\exp(1 + \eta \gamma_m^2 t))), \label{eq:gf_ub_inter}
\end{align}
where we used the choice $\vw_r = \vzero \implies b_r^m = 0$. Denoting $w = W(\exp(1 +
\eta \gamma_m^2 t))$, we have by the definition of $W$
\begin{align}
    w \exp(w) &= \exp(1 + \eta \gamma_m^2 t) \\
    w + \log w &= 1 + \eta \gamma_m^2 t \\
    2w &\Eqmark{i}{\geq} 1 + \eta \gamma_m^2 t \\
    w &\geq \frac{1 + \eta \gamma_m^2 t}{2},
\end{align}
where $(i)$ uses $\log w \leq w$. Plugging $w \geq \frac{1}{2} (1 + \eta \gamma_m^2 t)$ back into
\Eqref{eq:gf_ub_inter} yields $\widetilde{b}_r^m(t) \geq \log(\frac{1}{2}(1 + \eta \gamma_m^2 t))$,
and plugging this back into \Eqref{eq:gf_lemma_inter_2} yields
\begin{align}
    \beta_r^m &\leq \frac{1 + \exp(b_r^m)}{K} + \frac{1 + \exp(b_r^m)}{K} \sum_{k=1}^{K-1} \frac{1}{1 + \exp(\widetilde{b}_r^m(k))} \\
    &= \frac{2}{K} + \frac{2}{K} \sum_{k=1}^{K-1} \frac{1}{1 + \exp(\widetilde{b}_r^m(k))} \\
    &\leq \frac{2}{K} + \frac{4}{K} \sum_{k=1}^{K-1} \frac{1}{3 + \eta \gamma_m^2 k} \\
    &\leq \frac{2}{K} + \frac{4}{K} \int_0^{K-1} \frac{1}{3 + \eta \gamma_m^2 t} dt \\
    &= \frac{2}{K} + \frac{4}{\eta \gamma_m^2 K} \left[ \log(3 + \eta \gamma_m^2 t) \right]_0^{K-1} \\
    &= \frac{2}{K} + \frac{4}{\eta \gamma_m^2 K} \log \left( 1 + \frac{\eta \gamma_m^2 (K-1)}{3} \right) \\
    &\leq \frac{2}{K} + \frac{4}{\eta \gamma_m^2 K} \log \left( 1 + \frac{\eta \gamma_m^2 K}{3} \right) \\
    &\leq \frac{2}{K} + \frac{4}{\eta \gamma^2 K} \log \left( 1 + \frac{\eta \gamma^2 K}{3} \right),
\end{align}
where the last line uses $\gamma_m = \|\vx^m\| \geq \gamma$ together with the
fact that $f(x) = \log(1+x)/x$ is decreasing in $x$.
\end{proof}

\section{Additional Experimental Details} \label{app:experiment_details}
The synthetic and MNIST datasets that we use for the experiments in Section
\ref{sec:experiments} are described in full detail below.

\subsection{Synthetic Data}
The synthetic dataset is a simple task with $M=2$ clients and $n=1$ data points per
client, with $d=2$ dimensional data. It was introduced by \citet{crawshaw2025local} with
the goal of inducing conflict between the magnitude and direction of local client
updates. The two data points $\vx_1, \vx_2$ are defined in terms of parameters $\delta,
g$ as follows: $\vw_1 = \gamma_1 \vw_1^*$ and $\vw_2 = \gamma_2 \vw_2^*$, where
\begin{align}
    \vw_1^* &= \left( \frac{\delta}{\sqrt{1 + \delta^2}}, \frac{1}{\sqrt{1 + \delta^2}} \right) \\
    \vw_2^* &= \left( \frac{\delta}{\sqrt{1 + \delta^2}}, -\frac{1}{\sqrt{1 + \delta^2}} \right),
\end{align}
and $\gamma_1 = 1, \gamma_2 = 1/g$. By choosing $\delta$ close to zero and $g$ with
large magnitude, the two local objectives differ significantly in terms of gradient
direction and magnitude. For our experiments, we use $\delta = 0.1$ and $g = 10$.

\subsection{MNIST} \label{app:mnist_details}
Similar to \citet{wu2024large} and \citet{crawshaw2025local}, we use a subset of MNIST
data with binarized labels, and our implementation follows that of
\citet{crawshaw2025local}. First, we randomly select 1000 images from the MNIST dataset,
which we then partition among the $M$ clients using a heterogeneity protocol that is
common throughout the federated learning literature \citep{karimireddy2020scaffold}.
Specifically, for a data similarity parameter $s \in [0, 100]$, the $s\%$ of the data is
allocated to an ``iid pool", which is randomly shuffled, and a ``non-iid pool", which is
sorted by label. When sorting the non-iid pool, we sort according to the 10-way digit
label. We then split the iid pool into $M$ equally sized subsets, and similarly split
the non-iid pool into $M$ equally sized subsets (keeping the sorted order), and each
client's local dataset is comprised of one subset of the iid pool together with one
subset of the non-iid pool. In this way, the local datasets have different proportions
of each digit. If $s=100$, then the 1000 images are allocated uniformly at random to
different clients, and if $s=0$, then the clients will have nearly disjoint sets of
digits in their local datasets. Finally, after images have been allocated to clients, we
replace each image's label with the parity of its depicted digit. For our experiments,
we set $M=5$ and $s=50$. For all images, the pixel values initially fall into the range
$[0, 255]$; we normalize the data by subtracting $127$ from each pixel, then dividing
all pixels by the same scaling factor to ensure that $\max_{m,i} \|\vx_i^m\| = 1$.

\section{Additional Experimental Results}

\subsection{CIFAR-10 Experiments} \label{app:cifar_exp}
In this section, we provide additional experiments on the CIFAR-10 dataset, using
similar protocols as in Section \ref{sec:experiments}. For these experiments, we vary
the step size $\eta \in \{2^6, 2^7, \ldots, 2^{10}\}$, and other details of the setup
exactly match those of our MNIST experiments (see Section \ref{app:mnist_details}),
including the number of communication rounds $R$, the heterogeneity procedure, number of
clients $M$, number of samples per client $n$, data similarity parameter $s$, data
normalization procedure, and choice of interval $K \in \{1, 4, 16, 64\}$. Note that we
used step sizes between $2^6$ and $2^{10}$, since smaller choices led to very slow, very
stable convergence and larger choices led to overflow.

\begin{figure}
\begin{center}
\begin{subfigure}{0.5\linewidth}
    \includegraphics[width=\linewidth]{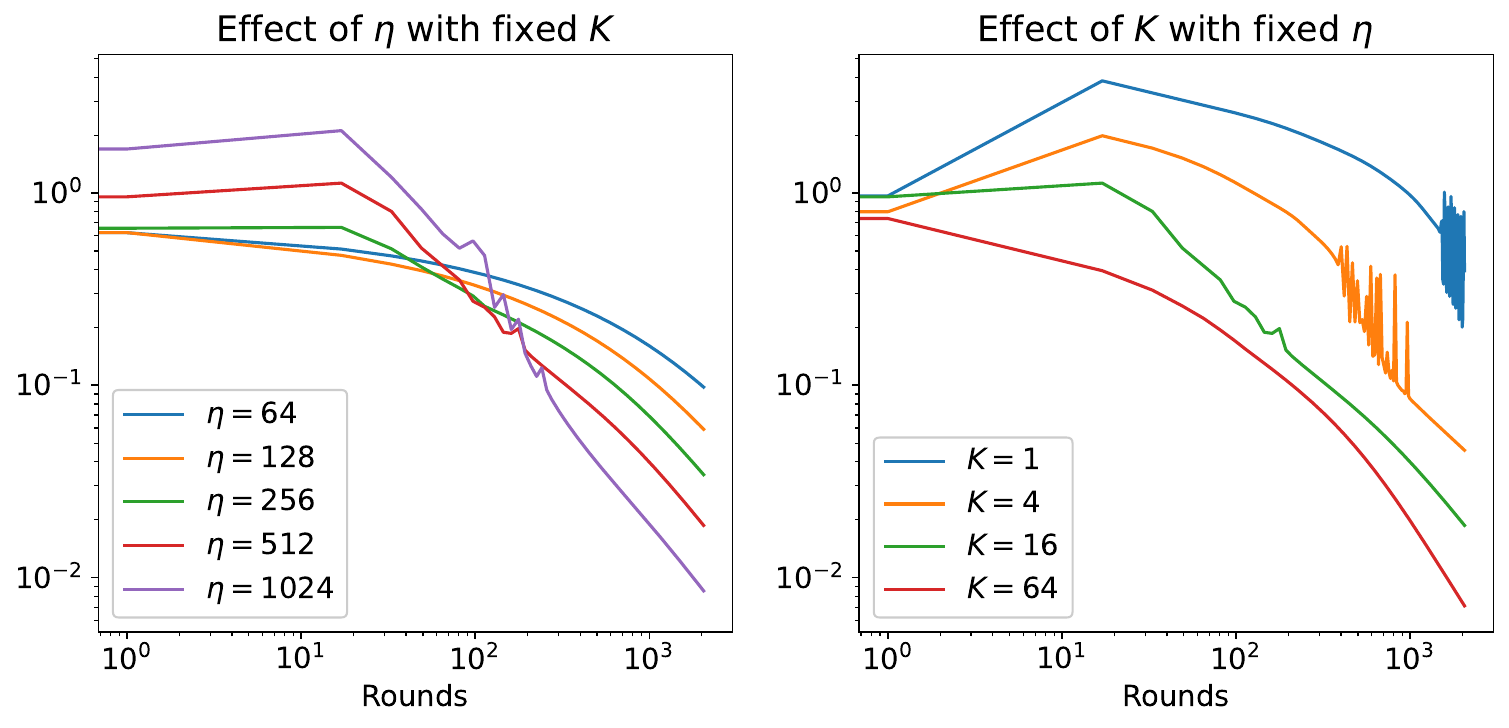}
    \caption{}
\end{subfigure}
\begin{subfigure}{0.225\linewidth}
    \includegraphics[width=\linewidth]{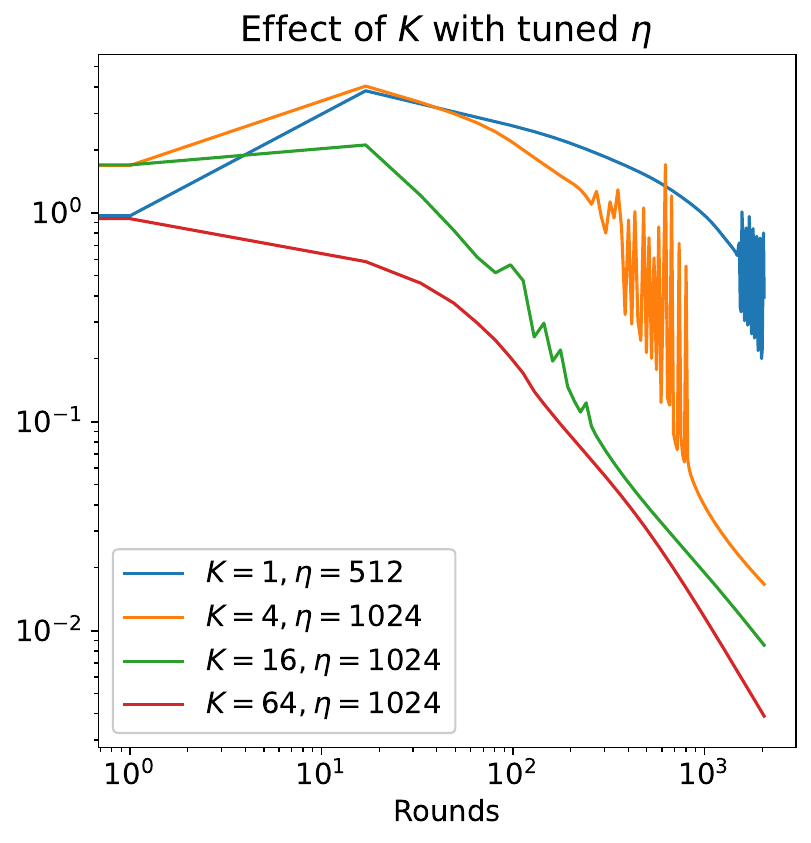}
    \caption{}
\end{subfigure}
\begin{subfigure}{0.225\linewidth}
    \includegraphics[width=\linewidth]{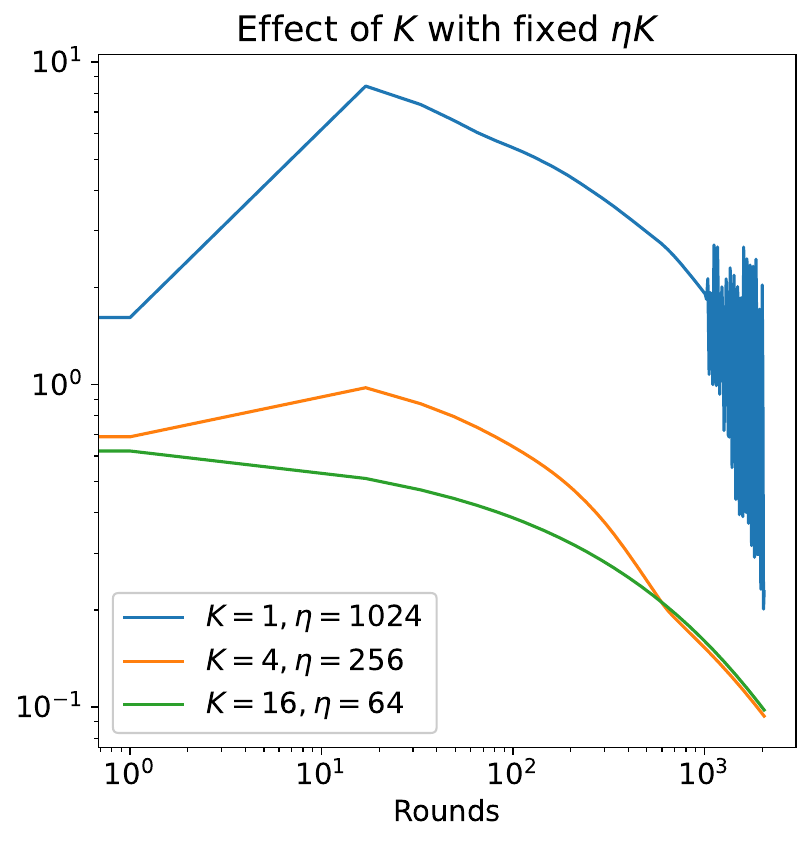}
    \caption{}
\end{subfigure}
\end{center}
\caption{
Train loss of Local GD (step size $\eta$, communication interval $K$) with the CIFAR-10
dataset. Overall, we observe that Local GD converges faster in the long run by choosing
a larger step size/communication interval, despite unstable/slow optimization in early
iterations. For $(a)$, we first fix $K = 16$ while varying $\eta$, then fix $\eta = 2^9$
while varying $K$.
}
\label{fig:cifar_results}
\end{figure}

The results can be seen in Figure \ref{fig:cifar_results}. For these additional
experiments, we used the same evaluation protocol as in Section \ref{sec:experiments}:
Figures \ref{fig:cifar_results}(a) corresponds to Q1 and Figure
\ref{fig:train_comparison}, Figure \ref{fig:cifar_results}(b) corresponds to Q2 and
Figure \ref{fig:tune_eta}, and Figure \ref{fig:cifar_results}(c) corresponds to Q3 and
Figure \ref{fig:constant_eta_K}.

The results on CIFAR-10 further support our theoretical findings. In Figure
\ref{fig:cifar_results}(a), larger step
sizes/communication intervals lead to faster convergence in the long run, despite the
resulting slow/unstable convergence in early iterations. In Figure
\ref{fig:cifar_results}(b), we can see that a larger communication interval $K$ leads to
faster convergence when $\eta$ is tuned to $K$. The results in Figure
\ref{fig:cifar_results}(c) are similar to the MNIST results in Figure
\ref{fig:constant_eta_K}: when $\eta K$ is constant, $K=1$ is less stable and slower
than other choices of $K$, and all other choices have roughly the same final loss. These
results strengthen the evidence that our theoretical findings accurately describe the
behavior of Local GD in practice.

\subsection{Margin Heterogeneity} \label{app:hetero_exp}
While our theoretical analysis makes no assumption about data heterogeneity (it applies
to any linearly separable dataset), the question remains whether the convergence rate
can be improved with a more fine-grained analysis that considers the local margins
$\gamma_m := \max_{\vw \in \mathbb{R}^d, \|\vw\|=1} \min_{(x, y) \in D_m} y \langle \vw,
\vw \rangle$ instead of the global margin $\gamma$ alone. We investigate this question
with a controlled synthetic dataset, by changing the local margins $\gamma_m$ while
preserving the global dataset.

\begin{figure}
\begin{center}
\includegraphics[width=0.32\linewidth]{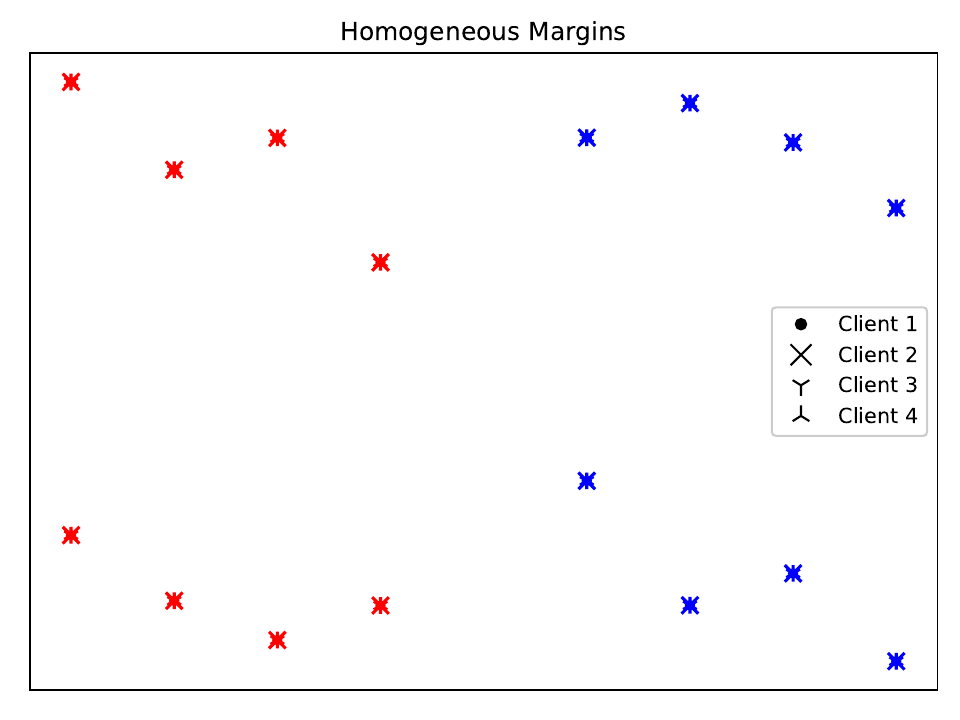}
\includegraphics[width=0.32\linewidth]{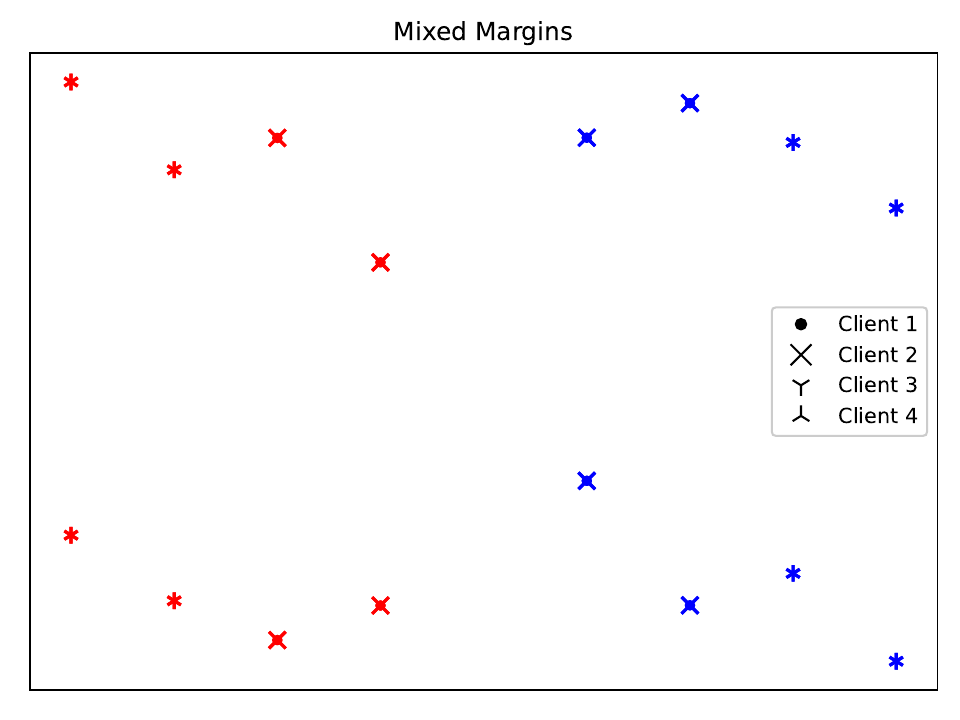}
\includegraphics[width=0.32\linewidth]{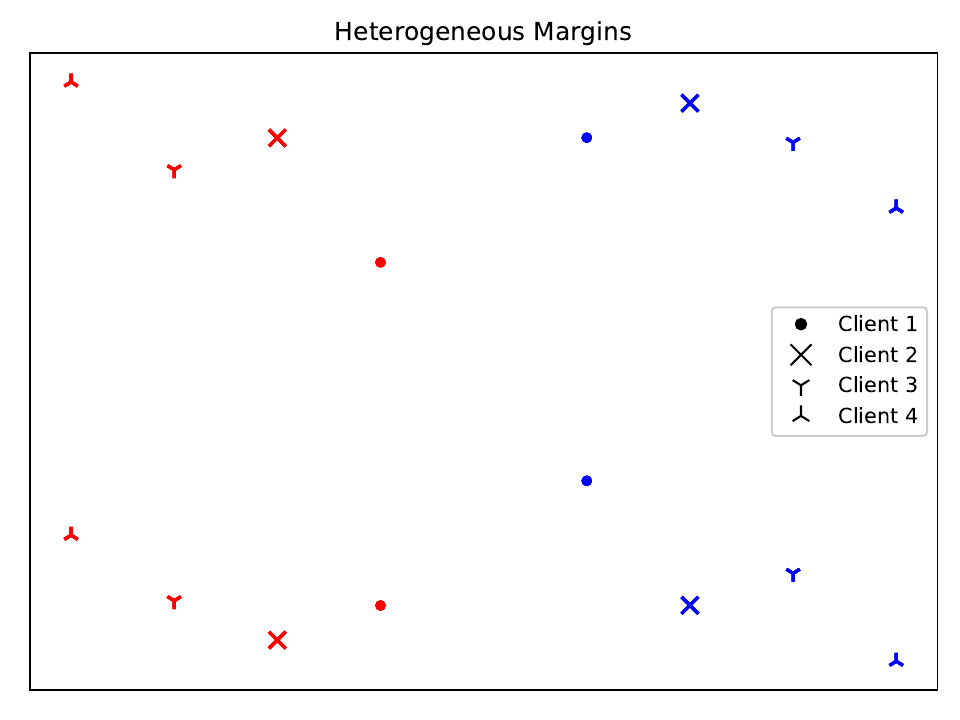}
\end{center}
\caption{
Three splits of a synthetic dataset. Binary labels are shown in red/blue, and client
indices for each data point are shown with markers. Note that some data points are
contained by multiple clients, which is shown with overlapping markers. In the
homogeneous split (left), all clients have the same data, so they all have the same
local margins. For mixed (middle), two clients have local margin $\gamma$, and two
clients have local margin $3 \gamma$. For heterogeneous (right), all four clients have
different local margins. Note that the combined dataset of all four clients is the same
for all three splits.
}
\label{fig:synthetic_splits}
\end{figure}

This synthetic dataset has $M = 4$ clients with a total of $16$ data points. The dataset
can be split among the four clients in three different ways to create either
homogeneous, partially homogeneous (i.e. mixed), or heterogeneous margins among clients,
which are shown in Figure \ref{fig:synthetic_splits}. Note that $\|\vx_i^m\| \leq 1$ for
every data point, so that $H \leq 1/4$, similarly with the datasets of Section
\ref{sec:experiments}. Also, the global dataset (and therefore $\gamma$) is the same for
all three splits. Our theory provides the same convergence rate upper bound for all
three splits, and we verify this prediction by evaluating Local GD with various
hyperparameters on the three splits. Results are shown in Figure
\ref{fig:synthetic_split_results}.

\begin{figure}
\begin{center}
\begin{subfigure}{0.6\linewidth}
    \includegraphics[width=\linewidth]{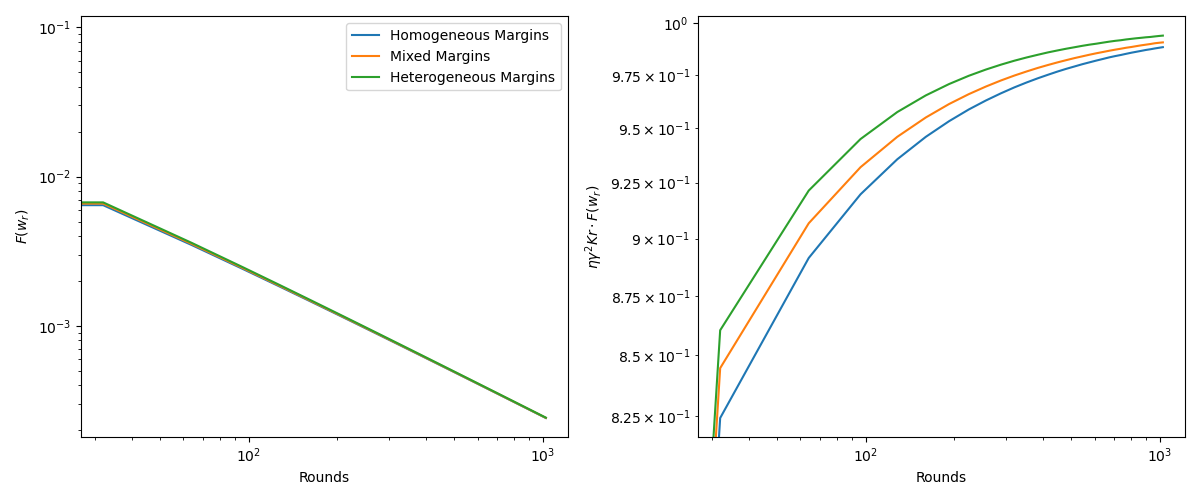}
    \caption{$\eta = 1, K = 16$}
\end{subfigure}
\begin{subfigure}{0.6\linewidth}
    \includegraphics[width=\linewidth]{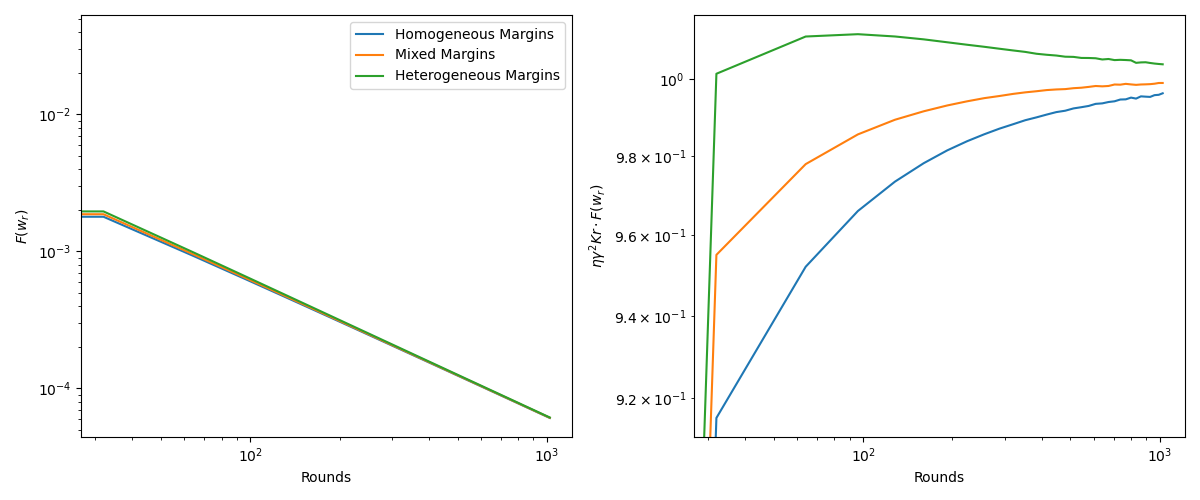}
    \caption{$\eta = 1, K = 64$}
\end{subfigure}
\begin{subfigure}{0.6\linewidth}
    \includegraphics[width=\linewidth]{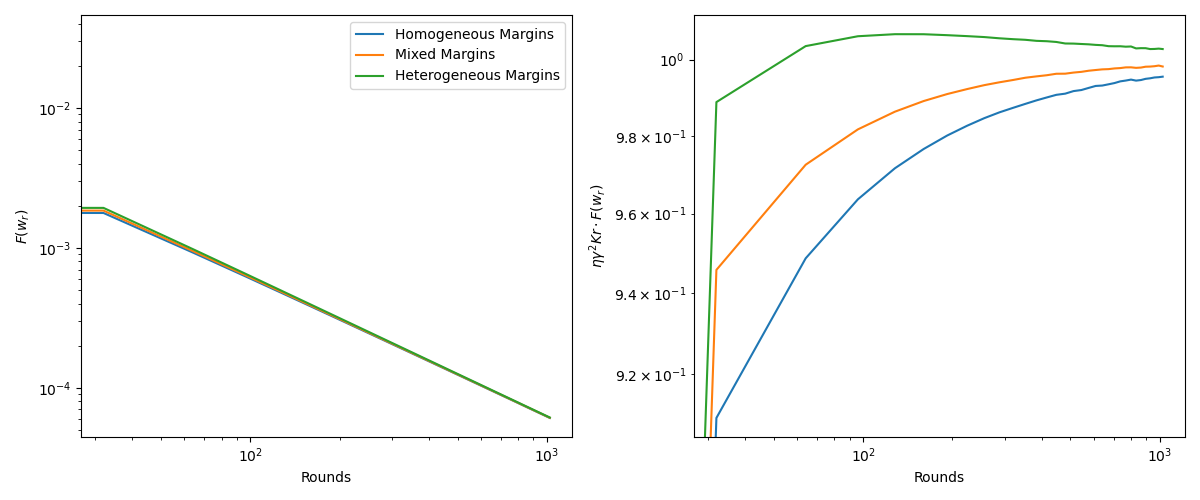}
    \caption{$\eta = 4, K = 16$}
\end{subfigure}
\begin{subfigure}{0.6\linewidth}
    \includegraphics[width=\linewidth]{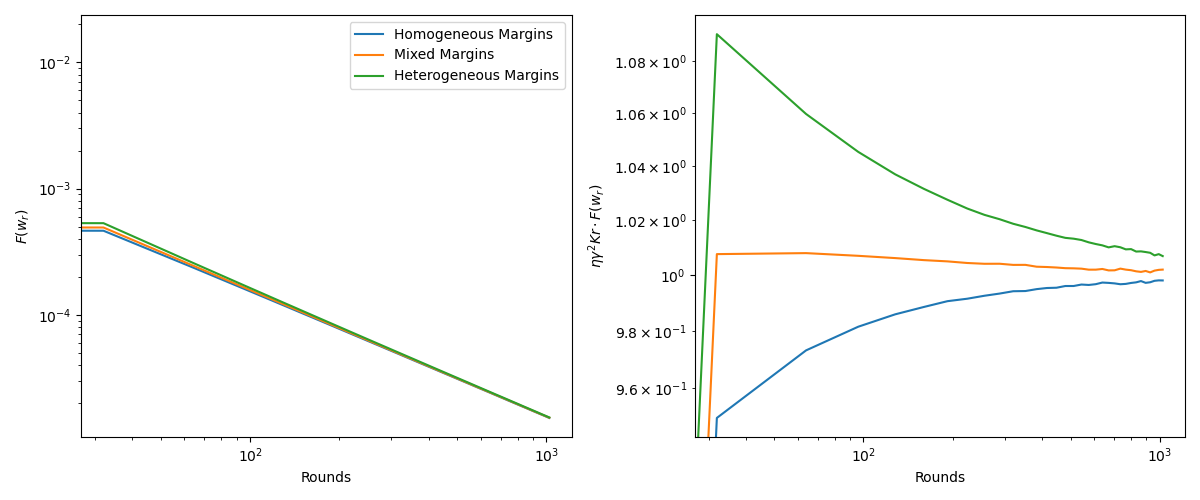}
    \caption{$\eta = 4, K = 64$}
\end{subfigure}
\end{center}
\caption{
Results of Local GD on three splits of the synthetic dataset pictured in Figure
\ref{fig:synthetic_splits}. The right subplots show the asymptotic rate as the number of
iterations goes to $\infty$, similarly to Figures 1(b) and 1(d) of \citep{wu2024large}.
}
\label{fig:synthetic_split_results}
\end{figure}

The left subplots of Figure \ref{fig:synthetic_split_results} show that the losses for
each split are slightly different in early iterations, but quickly become nearly
identical. The right subplots show that all three splits satisfy $\eta \gamma^2 K r
\cdot F(\vw_r) \rightarrow 1$ as $r$ increases, so that the asymptotic convergence rate
is unaffected by heterogeneity in the local margins. This behavior is consistent across
choices of $\eta$ and $K$. These results align with our theoretical prediction that the
convergence rate of Local GD depends on properties of the global dataset, rather than
how that dataset is allocated among clients.

%%%%%%%%%%%%%%%%%%%%%%%%%%%%%%%%%%%%%%%%%%%%%%%%%%%%%%%%%%%%%%%%%%%%%%%%%%%%%%%
%%%%%%%%%%%%%%%%%%%%%%%%%%%%%%%%%%%%%%%%%%%%%%%%%%%%%%%%%%%%%%%%%%%%%%%%%%%%%%%

\end{document}